%% file: main.tex
\documentclass[journal,twoside,web]{ieeecolor}
\usepackage{generic}
\usepackage{cite}
\usepackage{amsmath,amssymb,amsfonts,mathtools}

\usepackage{amsthm}

\usepackage[table]{xcolor}
\newtheorem{proposition}{Proposition}
\newtheorem{corollary}{Corollary}
\newtheorem{remark}{Remark}
\usepackage{algorithmic}
\usepackage{graphicx,verbatim}
\usepackage{algorithm,algorithmic}
\usepackage{hyperref}
\hypersetup{hidelinks=true}
\usepackage{textcomp}
\usepackage[T1]{fontenc}
\usepackage{booktabs,multirow}
\definecolor{Gray}{gray}{0.9}

\usepackage{array, makecell}
\definecolor{Gray}{gray}{0.9}

\usepackage{silence}
\usepackage{bm}  
\usepackage{microtype}
\usepackage{siunitx}
\usepackage{tabularx}

\usepackage[ruled,vlined,linesnumbered,algo2e]{algorithm2e}
\def\refname{References}
\usepackage{etoolbox}
\makeatletter
\patchcmd{\thebibliography}{\itemsep 0pt plus pt}{\itemsep 0pt plus 1pt}{}{}
\makeatother

\def\BibTeX{{\rm B\kern-.05em{\sc i\kern-.025em b}\kern-.08em
    T\kern-.1667em\lower.7ex\hbox{E}\kern-.125emX}}
\begin{document}
\title{SPARCL: Spatially Adaptive Regional Coverage for Active Learning in Chest X-ray Analysis}

\author{Juho Jung, Seongjae Kang, Mijoo Kim, Sungduk Cho, Yunkyoung Jun, Sunghoon Joo and Yoona Hwang$^{\dagger}$
\thanks{This research was supported by a grant of the Medical data-driven hospital support project through the Korea Health Information Service(KHIS), funded by the Ministry of Health \& Welfare, Republic of Korea (grant number: 2026-62).}
\thanks{J. Jung, S. Kang, M. Kim, S. Cho, Y. Jun, and S. Joo are with VUNO Inc., Seoul, South Korea (e-mail: \{juho.jung, seongjae.kang, mijoo.kim, sungduk.cho, yunkyoung.jun, sunghoon.joo\}@vuno.co).}
\thanks{$^{\dagger}$ Corresponding author. Y. Hwang is with VUNO Inc., Seoul, South Korea (e-mail: yoona.hwang@vuno.co).}
\thanks{Code: \href{https://github.com/vuno/sparcl}{\texttt{https://github.com/vuno/sparcl}}.}}

\maketitle

% This research was supported by a grant of the Medical data-driven hospital support project through the Korea Health Information Service(KHIS), funded by the Ministry of Health & Welfare, Republic of Korea.

\begin{abstract}
% Pixel-level annotation for chest X-ray (CXR) lesion segmentation is expensive, making active learning (AL) essential, yet existing AL strategies fail to account for the spatial distribution of selected samples, performing inconsistently across focal (e.g., aortic calcification) and diffuse (e.g., nodules) pathologies. In this paper, we propose \emph{SPARCL} (Spatially Adaptive Regional Coverage for Active Learning), which combines predictive uncertainty with cumulative spatial coverage: the spatial distribution of already-labeled lesions initializes a concave coverage state, so regions annotated in earlier rounds yield diminishing returns while under-labeled regions keep attracting the budget, and a gate driven by the positional diversity of uncertainty reverts \emph{SPARCL} to uncertainty sampling when there is no spatial variation to exploit. The objective is monotone submodular, so linear-time greedy batch selection is near-optimal. Because aggregate metrics such as Dice hide regional failures, we further introduce two spatial-reliability metrics, \emph{Spatial Inconsistency (SI)} and \emph{Performance Coverage (PC)}, that quantify how uniformly performance is distributed across anatomical regions and improves over AL rounds. On an in-house calcification cohort and public CXR segmentation and detection benchmarks, \emph{SPARCL} reaches 95\% of fully supervised performance with 24\% of the data on the in-house cohort versus 29\% for the best baseline, and consistently improves spatial reliability, surfacing regional gains that aggregate metrics miss.
Pixel-level annotation for chest X-ray (CXR) lesion segmentation is expensive, making active learning (AL) essential. However, existing AL strategies fail to account for the spatial distribution of selected samples, resulting in inconsistent performance across focal (e.g., aortic calcification) and diffuse (e.g., nodules) pathologies. In this paper, we propose \emph{SPARCL} (Spatially Adaptive Regional Coverage for Active Learning) to dynamically balance predictive uncertainty with cumulative spatial coverage. By employing a concave coverage function, \emph{SPARCL} imposes diminishing returns on previously annotated regions, naturally directing the annotation budget toward under-labeled areas. Furthermore, an entropy gate reverts the algorithm to standard uncertainty sampling when spatial variation is absent. This coverage objective is monotone submodular, guaranteeing near-optimal greedy batch selection. To address how aggregate metrics like Dice obscure regional failures, we further introduce two metrics---\emph{Spatial Inconsistency (SI)} and \emph{Performance Coverage (PC)}---to quantify anatomical fairness across AL rounds. Evaluated on an in-house calcification cohort and public CXR benchmarks, \emph{SPARCL} reaches 95\% of fully supervised performance using only 24\% of the data, compared to 29\% for the best baseline. Ultimately, \emph{SPARCL} ensures robust spatial reliability, surfacing regional gains that aggregate metrics miss.
\end{abstract}


\input{contents/1_introduction}
\input{contents/2_methods}

\input{contents/3_experiments}

 \input{contents/4_results}

\input{contents/5_conclusions}

\input{contents/6_data}

\appendices

\section{Proofs of Section~\ref{subsec:theory}}
\label{app:proofs}

\noindent\textbf{Lemma 1 (concave over modular).} \emph{Let $g:2^{V}\!\to\!\mathbb{R}_{\ge0}$ be non-negative modular, $g(\mathcal{S})=\sum_{x\in\mathcal{S}}w_x$ with $w_x\ge0$, and $\phi:\mathbb{R}_{\ge0}\!\to\!\mathbb{R}$ concave and non-decreasing. Then $\mathcal{S}\mapsto\phi(g(\mathcal{S}))$ is monotone and submodular.}

% \emph{Proof.} Monotonicity: for $\mathcal{A}\subseteq\mathcal{B}$, $g(\mathcal{A})\le g(\mathcal{B})$ (non-negative weights) and $\phi$ non-decreasing give $\phi(g(\mathcal{A}))\le\phi(g(\mathcal{B}))$. Submodularity: for $\mathcal{A}\subseteq\mathcal{B}$, $x\notin\mathcal{B}$, write $a=g(\mathcal{A})\le b=g(\mathcal{B})$ and $w=w_x\ge0$. Since $\phi$ is concave, $z\mapsto\phi(z+w)-\phi(z)$ is non-increasing, so $\phi(a+w)-\phi(a)\ge\phi(b+w)-\phi(b)$, i.e.\ the marginal gain of $x$ does not increase as the set grows. \hfill$\square$

\emph{Proof.} Monotonicity: for \(\mathcal{A}\subseteq\mathcal{B}\), \(g(\mathcal{A})\le g(\mathcal{B})\) (non-negative weights) and \(\phi\) non-decreasing give \(\phi(g(\mathcal{A}))\le\phi(g(\mathcal{B}))\). Submodularity: for \(\mathcal{A}\subseteq\mathcal{B}\), \(x\notin\mathcal{B}\), write \(a=g(\mathcal{A})\le b=g(\mathcal{B})\) and \(w=w_x\ge0\). Since \(\phi\) is concave, \(z\mapsto\phi(z+w)-\phi(z)\) is non-increasing, so \(\phi(a+w)-\phi(a)\ge\phi(b+w)-\phi(b)\). \hfill\(\square\)

\noindent\textbf{Proposition~\ref{prop:submod}.} Fix $m$ and set $\phi_m(z)=\psi(H_{0,m}+z)-\psi(H_{0,m})$ for $z\ge0$; since $\psi$ is concave and non-decreasing and $H_{0,m}\ge0$ is a constant, $\phi_m$ is concave, non-decreasing, and $\phi_m(0)=0$. The map $\mathcal{S}\mapsto\sum_{x\in\mathcal{S}}q_m(x)$ is non-negative modular ($q_m(x)\ge0$), so by Lemma~1 each $\mathcal{S}\mapsto\phi_m\big(\sum_{x\in\mathcal{S}}q_m(x)\big)=\psi(H_m(\mathcal{S}))-\psi(H_{0,m})$ is monotone submodular, and a finite sum of monotone submodular functions is monotone submodular, so $F$ is monotone submodular. The uncertainty term $\sum_{x\in\mathcal{S}}U(x)$ is modular with non-negative coefficients $U(x)\ge0$ (min--max normalized entropy), hence monotone submodular. As $\lambda^{(t)}\ge0$, $A=\sum_x U(x)+\lambda^{(t)}F$ is a non-negative combination of monotone submodular functions and is therefore monotone submodular, with $A(\varnothing)=\lambda^{(t)}\sum_m\phi_m(0)=0$. \hfill$\square$

% \noindent\textbf{Corollary~\ref{cor:greedy}.} $A$ is monotone submodular, normalized ($A(\varnothing)=0$). For maximization under the cardinality constraint $|\mathcal{S}|\le b$, greedy (repeatedly adding the element of largest marginal gain) achieves a $(1-1/e)$ approximation of the optimum~\cite{nemhauser1978analysis}. \hfill$\square$

\noindent\textbf{Corollary~\ref{cor:greedy}.} \(A\) is monotone submodular, normalized (\(A(\varnothing)=0\)). For maximization under the cardinality constraint \(|\mathcal{S}|\le b\), the greedy algorithm achieves a \((1-1/e)\) approximation of the optimum~\cite{nemhauser1978analysis}. \hfill\(\square\)

\noindent\textbf{Remark~\ref{prop:eps}.} For $\psi_\varepsilon(z)=(z+\varepsilon)^{p}-\varepsilon^{p}$ with $p\in(0,1)$, $\varepsilon>0$: $\psi_\varepsilon'(z)=p(z+\varepsilon)^{p-1}>0$ (non-decreasing) and $\psi_\varepsilon''(z)=p(p-1)(z+\varepsilon)^{p-2}<0$ (concave); $\psi_\varepsilon(0)=\varepsilon^{p}-\varepsilon^{p}=0$ and $\psi_\varepsilon'(0)=p\,\varepsilon^{p-1}<\infty$. Hence $\psi_\varepsilon$ satisfies the hypotheses of Proposition~\ref{prop:submod}, so $F$ and $A$ built from $\psi_\varepsilon$ remain monotone submodular and Corollary~\ref{cor:greedy} holds; and $\psi_\varepsilon(z)\to z^{p}$ pointwise as $\varepsilon\to0$, recovering $\sqrt{z}$ at $p=\tfrac12$. \hfill$\square$

\input{contents/references}
\end{document}

%% file: contents/1_introduction.tex
\section{Introduction}

Pixel-level annotation for chest X-ray (CXR) lesion segmentation is expensive and time-consuming, requiring radiologists to carefully delineate subtle anatomical structures across diverse pathological contexts~\cite{long2015fully,ronneberger2015u}. Crucially, model performance is not uniform; it often degrades in challenging areas such as the retrocardiac space or lung apices~\cite{wang2017chestx,irvin2019chexpert}. Active Learning (AL) addresses this by iteratively selecting the most informative samples for annotation, aiming to maximize model performance under limited labeling budgets~\cite{settles2009active,ren2021survey}. However, standard classification AL strategies~\cite{gal2016dropout,sener2017active,budd2021survey,nguyen2021goal} often fail in medical segmentation due to their inability to adapt to distinct spatial patterns~\cite{yang2017suggestive,mackowiak2018cereals,shi2024predictive}. The core challenge is not merely selecting ``hard'' samples, but selecting samples that provides the robust improvements across spatial regions~\cite{smailagic2018medal,nath2020diminishing}.

Specifically, we identify two opposing regional patterns arising from the spatially heterogeneous nature of CXR pathologies: 
a \emph{diffuse} pattern, where lesions scatter across distant anatomical regions (e.g., diffuse pulmonary nodules), and a \emph{focal} pattern, where pathology is spatially compact (e.g., focal aortic calcification, pneumothorax). However, existing AL strategies fail to account for this spatial structure: uncertainty-based methods~\cite{gal2017deep,kadir2023edgeal,shi2024predictive,tian2023fairseg,yu2023two,follmer2024active} often oversample redundant local neighborhoods while missing informative signals in distant regions~\cite{yang2017suggestive,gao2025deep,marinakis2024pulmonary,kirsch2019batchbald,ash2019deep,siddiqui2020viewal}. Diversity-based methods~\cite{sener2017active,ash2019deep,cai2021exploring,kirsch2019batchbald,hacohen2022active,yehuda2022active} promote image-level variety but remain agnostic to region-level spatial structure, offering no guarantee that selected batches cover underrepresented anatomical areas~\cite{yang2017suggestive,mackowiak2018cereals,gorriz2017cost}. Even naive region-balancing heuristics~\cite{mackowiak2018cereals,wang2016cost,xie2022towards,shin2021all} assume uniform information distribution across regions, wasting budget on uninformative normal regions.

% \begin{figure*}[t]
%     \centering
%     \includegraphics[width=\textwidth]{figures/fig1_sparcl_ver5.pdf}
%     \caption{\textbf{Overall framework of \emph{SPARCL}.} Each unlabeled image yields an uncertainty score $U(x)$ and a spatial distribution $\mathbf{q}(x)$ of that uncertainty over grid bins. The coverage state starts from the bins occupied by already-labeled lesions ($H_0$); a gate $\lambda^{(t)}$ set by the positional diversity of uncertainty across the candidates closes only when that uncertainty concentrates in a few bins. Greedy marginal-gain maximization then builds the batch one image at a time, rewarding uncertainty and coverage of bins still underrepresented by the labeled set and by this round's picks, so coverage accumulates across rounds.}
%     \label{fig:sparcl_framework}
% \end{figure*}

\begin{figure*}[t]
    \centering
    \includegraphics[width=\textwidth]{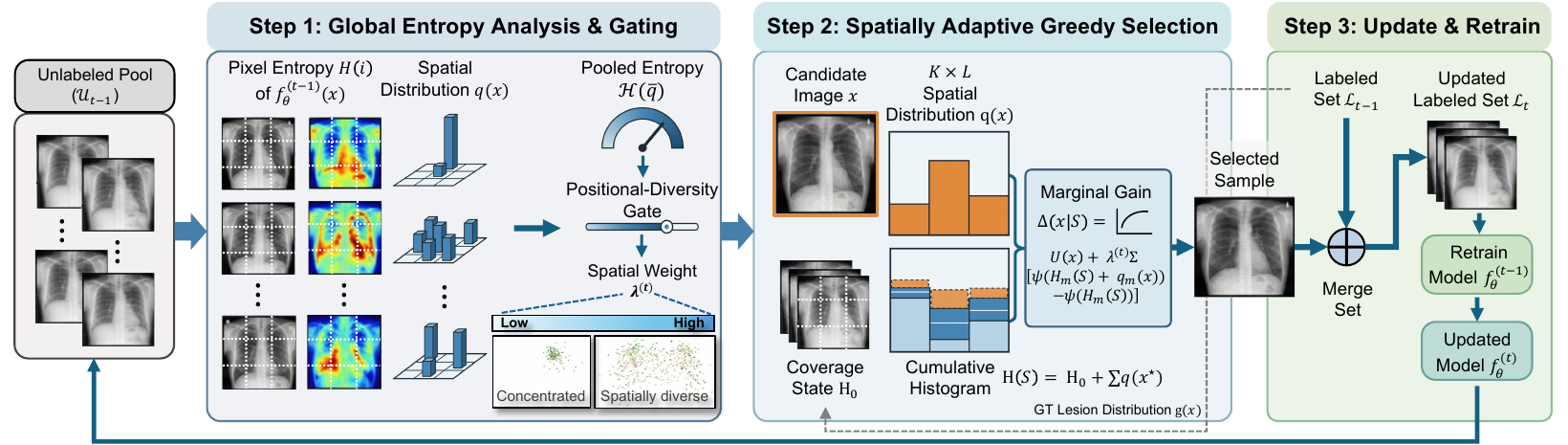}
    \caption{\textbf{Overall framework of \emph{SPARCL}.} 
    \textbf{Step~1:} Each unlabeled image \(x\in\mathcal{U}_{t-1}\) yields an uncertainty score \(U(x)\) and a spatial distribution \(q(x)\) of its pixel entropy. Over the top-\(rB\) candidates, the entropy \(\mathcal{H}(\bar{q})\) of their pooled distribution sets the positional-diversity gate \(\lambda^{(t)}\), lowering the spatial weight when the candidates' uncertainty concentrates in the same few bins.
    \textbf{Step~2:} Rather than starting from zero, the coverage state \(H(\mathcal{S})\) is initialized to \(H_0\), the ground-truth lesion distribution \(g(x)\) of the labeled set. Greedy marginal-gain maximization then adds one image at a time, updating \(H(\mathcal{S})\) after each pick, so that it rewards uncertainty and bins underrepresented by both the labeled set and the current batch.
    \textbf{Step~3:} The selected samples are annotated, merged into \(\mathcal{L}_t\), and used to retrain the model.}
    \label{fig:sparcl_framework}
\end{figure*}

To bridge this gap, we propose \emph{\textbf{SPARCL}} (\textbf{Sp}atially \textbf{A}daptive \textbf{R}egional \textbf{C}overage for Active \textbf{L}earning). Unlike static policies, \emph{SPARCL} balances predictive uncertainty with spatial coverage that prioritizes regions underrepresented in the labeled set. Coverage is accumulated across rounds, so regions in which lesions have already been annotated yield diminishing returns while under-labeled regions keep attracting the budget; a gate driven by the positional diversity of uncertainty across the most uncertain candidates switches coverage off only when there is no spatial variation to exploit, in which case \emph{SPARCL} reduces to uncertainty sampling. This allows a unified framework to handle \emph{focal} and \emph{diffuse} pathologies without manual heuristic tuning.

Furthermore, we argue that aggregate metrics like Dice or IoU are insufficient for AL evaluation for CXR segmentation, as they obscure regional biases~\cite{taha2015metrics,sudre2017generalised,reinke2021common,maier2024metrics}. We introduce two region-level metrics that quantify spatial reliability: \textbf{Spatial Inconsistency (SI)}, which measures performance disparity across anatomical regions, and \textbf{Performance Coverage (PC)}, which captures consistent improvements across anatomical regions between AL rounds. We summarize the key contributions as follows.

\begin{enumerate}
    \item We demonstrate that CXR pathologies exhibit two distinct spatial regimes (\emph{focal} vs.\ \emph{diffuse}), and reveal how existing AL methods fail to maintain consistent performance across these diverse anatomical patterns.
    \item We propose \textbf{\textit{SPARCL}}, an entropy-gated acquisition strategy that dynamically balances predictive uncertainty with cumulative spatial coverage. Its coverage objective is rigorously shown to be monotone submodular, guaranteeing near-optimal greedy batch selection while seamlessly adapting to both focal and diffuse regimes.
    \item We introduce two novel region-level evaluation metrics---\textbf{\textit{Spatial Inconsistency} (SI)} and \textbf{\textit{Performance Coverage} (PC)}---to rigorously quantify anatomical ``\textit{blind spots}.'' These provide a new standard to evaluate spatial fairness across any AL acquisition strategy, exposing regional biases obscured by standard aggregate metrics.
\end{enumerate}

\section{Related Work}

\subsection{Active learning for medical image segmentation}
Uncertainty sampling selects images on which the model is least confident, using predictive entropy or MC-dropout/BALD estimates~\cite{gal2016dropout,gal2017deep,yang2017suggestive}, learned loss prediction~\cite{yoo2019learning}, or edge/region-aware cues~\cite{kadir2023edgeal,siddiqui2020viewal}. Representativeness-based methods instead promote batch diversity through core-set~\cite{sener2017active}, gradient embeddings~\cite{ash2019deep}, or hybrid criteria~\cite{cai2021exploring,budd2021survey,ren2021survey,gaillochet2023active}. None of these criteria tracks \emph{where} in the image evidence has already been labeled.
\vspace{-5pt}

\subsection{Batch and submodular acquisition}
Submodularity prevents redundant batch selections by guaranteeing a greedy $(1-1/e)$ approximation~\cite{nemhauser1978analysis}. Prior works apply submodularity to batch mutual information~\cite{kirsch2019batchbald,follmer2024active} or Gaussian random fields~\cite{ma2013submodularity} in feature space. Coverage has also been cast in embedding space~\cite{yehuda2022active,hacohen2022active}; SPARCL instead applies the concave-over-modular reward of Lin and Bilmes~\cite{lin2011class} to a \emph{spatial} histogram accumulated over the labeled lesions, inheriting the $(1-1/e)$ guarantee.
% Coverage has also been cast in embedding space: ProbCover~\cite{yehuda2022active} maximizes the number of pool points covered by the selected set and TypiClust~\cite{hacohen2022active} selects typical points per cluster. SPARCL instead applies the classical concave-over-modular diversity reward of Lin and Bilmes~\cite{lin2011class} to a \emph{spatial} histogram that accumulates the labeled set's lesion locations across rounds, so the $(1-1/e)$ guarantee is inherited rather than new.

\subsection{Evaluating spatial reliability}
Aggregate metrics like Dice and IoU summarize an entire image but mask large regional disparities~\cite{taha2015metrics,sudre2017generalised,reinke2021common}. Metrics Reloaded~\cite{maier2024metrics} formalizes these pitfalls and recommends complementary metrics that expose localized failures hidden by aggregate overlap; in that spirit, and alongside fairness-oriented evaluation~\cite{tian2023fairseg}, SI and PC are region-level complements to Dice---not replacements---that expose anatomical blind spots and quantify spatial learning uniformity across AL rounds.

%% file: contents/2_methods.tex
\section{Methods}

\subsection{Problem Setup and Notation}
Let $\mathcal{U}_0 = \{x_j\}_{j=1}^{N}$ be an initial unlabeled pool of $N$ CXR images, and $\mathcal{L}_0 = \{ (x_i, y_i) \}_{i=1}^{M_0}$ be a seed set of $M_0$ labeled images with binary masks $y_i \in \{0, 1\}^{H \times W}$. 
We train a segmentation network $f_\theta$ to minimize the supervised loss $L_{\text{seg}}$. 
In each active learning (AL) round $t$, a batch of unlabeled images $\mathcal{B}_t \subset \mathcal{U}_{t-1}$ of size $B$ is selected by maximizing a set-level acquisition function $A(\mathcal{S})$ (Eq.~\eqref{eq:sparcl_value}).
These selected samples are then annotated by an oracle (e.g., radiologist) to obtain their corresponding masks $y$, and both sets are updated: $\mathcal{L}_t \leftarrow \mathcal{L}_{t-1} \cup \{(x, y) \mid x \in \mathcal{B}_t\}$ and $\mathcal{U}_t \leftarrow \mathcal{U}_{t-1} \setminus \mathcal{B}_t$.
To evaluate \textit{spatial reliability}, we partition the image domain $\Omega$ into $M$ disjoint spatial bins $\{\Omega_m\}_{m=1}^M$.

% Figure 1: Performance Curve
\begin{figure*}[t]
  \centering
  \includegraphics[width=\textwidth]{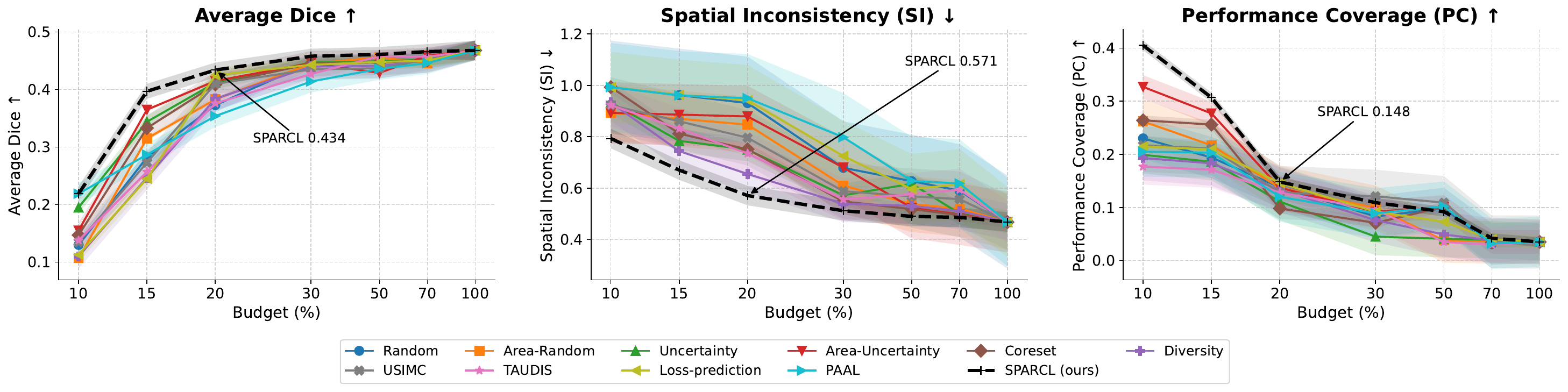}
  \caption{\textbf{Active learning performance across annotation budgets (10--100\%)}. \emph{SPARCL} reaches the highest Dice($\uparrow$) and lowest SI($\downarrow$) at every budget up to 50\% and the highest PC($\uparrow$) at $\le 20\%$. PC is a per-round gain and decays for all methods as learning saturates.} 
  \label{fig:alc_budget_curves}
\end{figure*}

% Figure 1: Performance Curve
% \begin{figure}[t]
%   \centering
%   \includegraphics[width=\columnwidth]{figures/alc_dice_si_1col_col.pdf}
%   \caption{\textbf{Active learning performance across annotation budgets (10--100\%)}. \emph{SPARCL} achieves the ideal learning trajectory by simultaneously maximizing label efficiency (Dice $\uparrow$) while minimizing spatial inconsistency ($SI$ $\downarrow$) compared to all baselines.}
%   \label{fig:alc_budget_curves}
% \end{figure}

\subsection{SPARCL: Spatially Adaptive Regional Coverage}
As outlined in Fig.~\ref{fig:sparcl_framework}, \emph{SPARCL} balances predictive uncertainty with \emph{cumulative} spatial coverage. Each round scores every unlabeled image by how uncertain it is and \emph{where} that uncertainty lies, restricts selection to the most uncertain candidates, initializes a coverage state from the bins that already-labeled lesions occupy, and greedily selects a batch that is both uncertain and complementary to what has been labeled; a gate $\lambda^{(t)}$ switches coverage off only when the candidates' uncertainty has no spatial variation to exploit.

% \subsubsection{Predictive Uncertainty Estimation.}
% We estimate pixel-wise uncertainty via $T$ stochastic forward passes (MC Dropout). The model outputs a probability map, so each pass $t = 1, \ldots, T$ gives $p^{(t)}(i) \in [0,1]$ for pixel $i$.
% The \emph{mean probability} is $\bar{p}(i) = \frac{1}{T} \sum_{t=1}^{T} p^{(t)}(i)$. We compute the predictive entropy as $H(i) = -\bar{p}(i) \log \bar{p}(i) - (1 - \bar{p}(i)) \log (1 - \bar{p}(i))$.
% The image-level uncertainty $U(x)$ is defined as the mean entropy across the pixels: $U(x) = \frac{1}{|\Omega|} \sum_{i \in \Omega} H(i)$.

\subsubsection{Predictive Uncertainty Estimation.}
Let $p(i)\in[0,1]$ be the foreground probability that $f_\theta$ assigns to pixel $i$. Its predictive entropy is $H(i) = -p(i) \log p(i) - (1 - p(i)) \log (1 - p(i))$, and the image-level uncertainty is the mean entropy over pixels, $U(x) = \frac{1}{|\Omega|} \sum_{i \in \Omega} H(i)$. Any stochastic estimate of $p(i)$ (e.g., MC Dropout or test-time augmentation) can be substituted without changing what follows.

\subsubsection{Uncertainty Spatial Histogram.}  
For each image $x$, we construct a spatial histogram $\mathbf{h}(x) \in \mathbb{R}^M$ with $M = KL$ for a $K \times L$ grid. The $m$-th entry is $h_m(x) = \sum_{i \in \Omega_m} H(i)$, aggregating pixel entropy within bin $m$. 
We normalize it per image, $q_m(x) = h_m(x)/\sum_{j} h_j(x)$, so that $\mathbf{q}(x)$ is a distribution over bins that describes \emph{where} the image's uncertainty lies irrespective of its total amount, which $U(x)$ already captures. Both the gate and the coverage term below operate on $\mathbf{q}(x)$.

\subsubsection{Candidate Set (Uncertainty Floor).}
Coverage should redirect the budget \emph{among} informative images, not admit uninformative ones. Each round is therefore restricted to the candidate set $\mathcal{U}'_{t-1}\subseteq\mathcal{U}_t$ of the $rB$ images with the largest $U(x)$ ($r=5$ throughout), over which $U(x)$ is min--max normalized to $[0,1]$; whatever its coverage gain, an image outside the top-$rB$ by uncertainty is never selected. All quantities below are computed on $\mathcal{U}'_t$.

\subsubsection{Cumulative Coverage State.}
For every lesion-bearing labeled image $x\in\mathcal{L}^{+}_{t-1}$ we form the analogous distribution of its \emph{ground-truth} lesion over bins, $g_m(x)=|\{i\in\Omega_m:\,y(i)=1\}|\,/\,|\{i\in\Omega:\,y(i)=1\}|$, and initialize the coverage state of round $t$ as
\begin{equation}
\vspace{-5pt}
H_{0,m}=\frac{B}{|\mathcal{L}^{+}_{t-1}|}\sum_{x\in\mathcal{L}^{+}_{t-1}} g_m(x),\qquad m=1,\dots,M,
\label{eq:h0}
\end{equation} 
the labeled set's lesion distribution rescaled to the mass of one batch. $H_{0}$ records \emph{where lesions have already been annotated} and $\mathbf{q}(x)$ \emph{where the model is now uncertain}; both are per-image distributions over the same bins, so they accumulate in one histogram, and bins already dense in the labeled set yield small concave gains while under-labeled bins keep a large gain from the first selection of the round. Within a round the state advances by the predicted $\mathbf{q}(x)$ of each selected image, but it is rebuilt from ground truth at the next round, so cross-round memory is carried by annotations and estimation noise in $\mathbf{q}$ does not accumulate.

\input{tables/alg_1}

\input{tables/table_template_table1_seg}
\input{tables/table_template_table3_spatial_reliability_metric}

\subsubsection{Positional-Diversity Gate.}
\label{sec:gate}
Cumulative coverage requires spatial variation across images, so we gate the coverage weight by the entropy of the \emph{pooled} uncertainty distribution over bins,
\begin{equation}
\bar{q}_m=\frac{1}{|\mathcal{U}'_t|}\sum_{x\in\mathcal{U}'_t} q_m(x),\qquad
\lambda^{(t)}=\lambda_{\max}\,\frac{\mathcal{H}(\bar{\mathbf{q}})}{\log M}\in[0,\lambda_{\max}],
\label{eq:gate}
\end{equation}
with \(\mathcal{H}(\bar{\mathbf{q}})=-\sum_m\bar{q}_m\log\bar{q}_m\). This entropy is high when uncertainty spans many bins (e.g., diffuse lesions, or focal lesions varying across patients). It is low only when uncertainty consistently concentrates in the same few bins (e.g., a single-site pathology). In this regime, enforcing coverage would merely select spurious background noise, so the gate closes, reverting \emph{SPARCL} to pure uncertainty sampling. We use the pooled rather than mean per-image entropy because, by Jensen's inequality, their difference \(\mathcal{H}(\bar{\mathbf{q}})-\frac{1}{|\mathcal{U}'_t|}\sum_x\mathcal{H}(\mathbf{q}(x))\) captures the between-patient spatial variation---precisely the diversity our coverage exploits.

% Figure 2: Sample Selection
\begin{figure*}[t]
    \centering
    \includegraphics[width=0.95\textwidth]{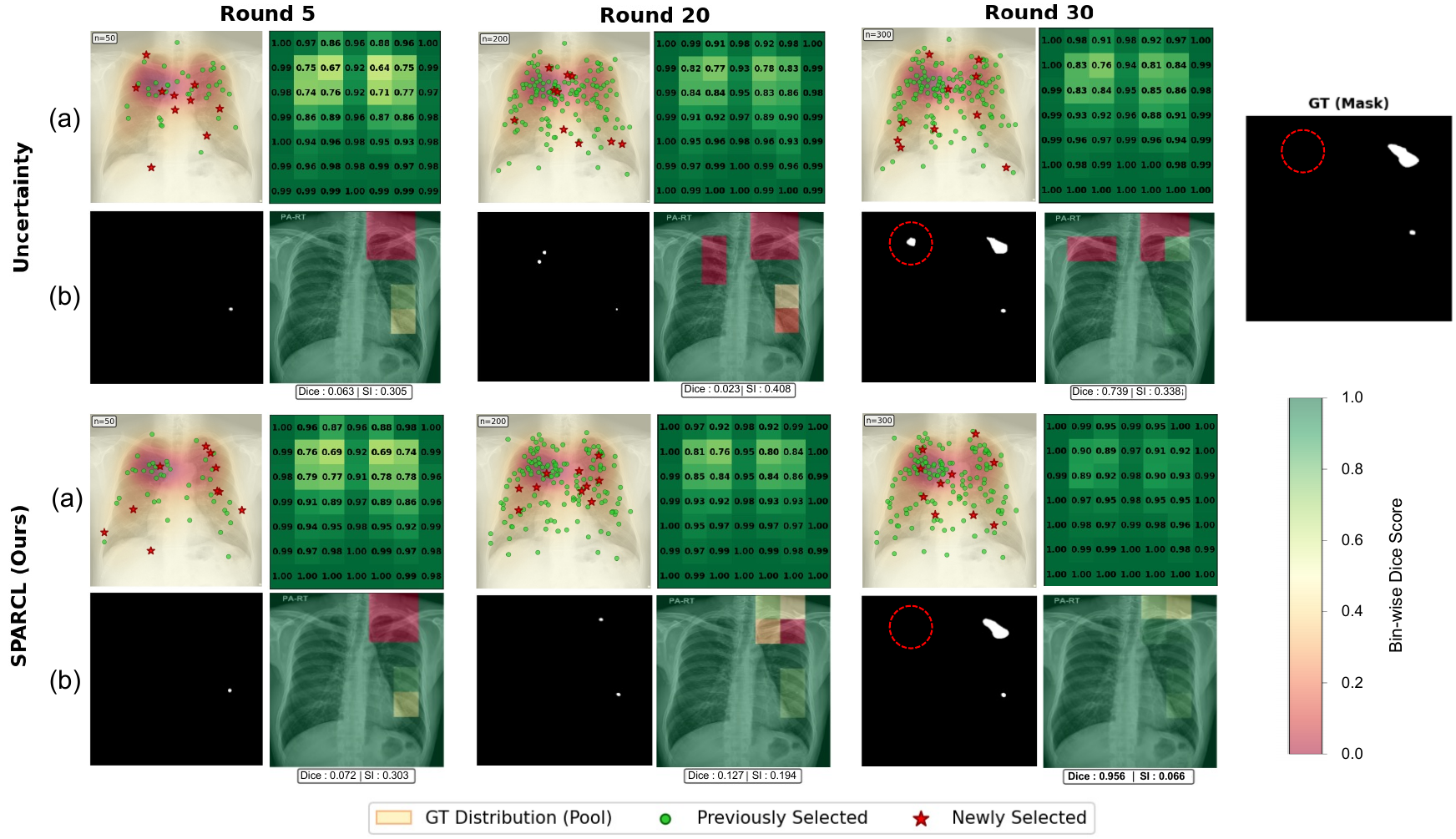}
    \caption{\textbf{Acquisition behavior and its spatial consequences} (in-house; rows: \textit{Uncertainty} vs.\ \emph{SPARCL}; columns: rounds 5/20/30).
    \textbf{(a)} Queries over the pool lesion density (\textcolor{green}{$\bullet$}: earlier rounds; \textcolor{red}{$\star$}: this round) beside the $7{\times}7$ per-bin test Dice grid (red$\rightarrow$green as Dice $0\rightarrow1$). The upper-lung-zone bins (2nd--3rd rows), where pool lesions are densest, remain ``blind spots'' under \textit{Uncertainty} (Dice $0.76$--$0.86$ at round~30), whereas \emph{SPARCL} raises them to $0.89$--$0.93$.
    \textbf{(b)} The \textcolor{red}{red dashed circle} marks a lesion-free region in the right upper lung zone of a test case: at round~30 \textit{Uncertainty} places a false positive there (Dice~$0.739$, SI~$0.338$) while \emph{SPARCL} leaves it empty (Dice~$0.956$, SI~$0.066$).}
    \label{fig:selection_map}
\end{figure*}

\subsubsection{Acquisition Objective.} We employ a greedy strategy with a concave gain $\psi$, by default $\psi(z)=\sqrt{z}$~\cite{lin2011class} (Remark~\ref{prop:eps} covers the regularized family). The marginal gain of adding image $x$ to the selected set $\mathcal{S}$ is:
\begin{equation}
\Delta(x | \mathcal{S}) = U(x) + \lambda^{(t)} \sum_{m=1}^{M} \left[ \psi(H_m(\mathcal{S}) + q_m(x)) - \psi(H_m(\mathcal{S})) \right]
\label{eq:sparcl_obj}
\end{equation}
where $H_m(\mathcal{S}) = H_{0,m}+\sum_{x' \in \mathcal{S}} q_m(x')$ is the coverage mass in bin $m$ contributed by the labeled set (Eq.~\eqref{eq:h0}) and by the images already selected in this round. Both terms are bounded: $U(x)\in[0,1]$ after normalization and, since each $\mathbf{q}(x)$ sums to one and $\psi$ is concave with $\psi(0)=0$, $\sum_m[\psi(H_m+q_m(x))-\psi(H_m)]\le\sum_m\psi(q_m(x))\le\sqrt{M}$, with equality only for an empty coverage state; the gain shrinks as bins fill, so $\lambda_{\max}$ sets the relative weight of the two terms on a fixed scale. Equation~\eqref{eq:sparcl_obj} is the marginal gain of the set function
\begin{equation}
A(\mathcal{S}) = \underbrace{\sum_{x \in \mathcal{S}} U(x)}_{\text{uncertainty (modular)}} \;+\; \lambda^{(t)} \underbrace{\sum_{m=1}^{M} \Big[\psi\!\big(H_m(\mathcal{S})\big)-\psi\!\big(H_{0,m}\big)\Big]}_{\text{coverage } F(\mathcal{S})},
\label{eq:sparcl_value}
\end{equation}
so that $\Delta(x\mid\mathcal{S}) = A(\mathcal{S}\cup\{x\}) - A(\mathcal{S})$, and each round selects a batch from the candidate set $\mathcal{U}'_t$ by greedily maximizing $A$ subject to $|\mathcal{S}|\le B$ (\textbf{Algorithm~\ref{alg:sparcl}}). Sec.~\ref{subsec:theory} shows this greedy rule inherits a constant-factor optimality guarantee.

\subsection{Theoretical Guarantees}
\label{subsec:theory}
% Casting acquisition as maximizing the set function $A$ of Eq.~\eqref{eq:sparcl_value} yields a standard approximation guarantee for the greedy schedule. While concave-over-modular coverage~\cite{lin2011class} and submodular batch acquisition in information or feature spaces~\cite{kirsch2019batchbald,ma2013submodularity} are established, our contribution applies this guarantee to a gated, \emph{cumulative spatial-coverage} objective. Full proofs for all theoretical claims are provided in Appendix~\ref{app:proofs}.
Casting acquisition as maximizing the set function $A$ of Eq.~\eqref{eq:sparcl_value} yields a standard approximation guarantee for the greedy schedule; proofs are in Appendix~\ref{app:proofs}.

\begin{proposition}[Monotone submodularity]
\label{prop:submod}
Let $\psi:\mathbb{R}_{\ge0}\!\to\!\mathbb{R}_{\ge0}$ be concave and non-decreasing with $\psi(0)=0$, let $q_m(x)\ge0$, and let $H_{0,m}\ge0$ be fixed. Then the coverage term $F(\mathcal{S})=\sum_{m=1}^{M}\big[\psi\big(H_{0,m}+\sum_{x\in\mathcal{S}}q_m(x)\big)-\psi(H_{0,m})\big]$ is monotone and submodular, and for any gate $\lambda^{(t)}\ge0$ the acquisition value $A(\mathcal{S})$ of Eq.~\eqref{eq:sparcl_value} is monotone submodular with $A(\varnothing)=0$.
\end{proposition}

\begin{corollary}[Greedy near-optimality]
\label{cor:greedy}
As $A$ is monotone submodular, normalized, and non-negative, the greedy rule of Algorithm~\ref{alg:sparcl} selecting a batch of size $B$ returns $\mathcal{S}_{\mathrm{greedy}}$ with $A(\mathcal{S}_{\mathrm{greedy}})\ge(1-1/e)\,A(\mathcal{S}^{\star})$, where $\mathcal{S}^{\star}$ is the optimal size-$B$ subset of the candidate set $\mathcal{U}'_t$ from which the batch is drawn~\cite{nemhauser1978analysis}. Since $H_0$ and $\lambda^{(t)}$ are fixed within a round, the bound applies round by round.
\end{corollary}

\noindent The guarantee concerns the acquisition objective $A$, not downstream Dice: it certifies that the linear-time greedy rule near-optimally solves an otherwise combinatorial batch-selection problem.

\begin{remark}[Regularized coverage removes the empty-bin singularity]
\label{prop:eps}
The default gain $\psi(z)=\sqrt{z}$ has $\psi'(0^{+})=+\infty$. Every gain actually evaluated is nevertheless bounded, $\psi(H+q)-\psi(H)\le\psi(q)\le1$ since $q_m(x)\le1$, and the $\varepsilon$-power gain $\psi_\varepsilon(z)=(z+\varepsilon)^{p}-\varepsilon^{p}$, $p\in(0,1)$, removes the singularity itself ($\psi_\varepsilon'(0)=p\,\varepsilon^{p-1}<\infty$) while remaining concave, non-decreasing, and $\psi_\varepsilon(0)=0$; it therefore leaves Proposition~\ref{prop:submod} and Corollary~\ref{cor:greedy} intact and recovers $\sqrt{z}$ as $\varepsilon\!\to\!0$ (Appendix~\ref{app:proofs}).
\end{remark}

% Crucially, the gate $\lambda^{(t)}$ and the coverage state $H_0$ depend only on quantities fixed before the greedy step (the candidate set and the labeled set), so each round maximizes a fixed monotone submodular function and the $(1-1/e)$ guarantee holds round by round regardless of how $\lambda^{(t)}$ and $H_0$ evolve across rounds. Furthermore, while the default gain $\psi(z)=\sqrt{z}$ has an unbounded derivative at zero ($\psi'(0^{+})=+\infty$), every marginal gain that is actually evaluated is bounded, $\psi(H+q)-\psi(H)\le\psi(q)\le1$ because $q_m(x)\le1$. A formal $\varepsilon$-power regularization $\psi_\varepsilon(z)$ that removes the singularity itself while preserving submodularity is detailed in Appendix~\ref{app:proofs}.

\subsection{Spatial Reliability Metrics}
\label{subsec:metrics}
% Aggregate metrics like Dice can mask regional biases, excelling on frequent central findings while failing in underrepresented areas. To quantify spatial uniformity, we evaluate performance across $M$ disjoint spatial bins $\{\Omega_m\}_{m=1}^M$. Because naive per-bin Dice is undefined in empty regions, we adopt a per-image, \emph{active-bin}, \emph{unweighted} formulation. While global and worst-bin Dice capture aggregate and lower-bound accuracy, they ignore spatial dispersion. Therefore, we introduce \textbf{SI} to measure cross-region disparity and \textbf{PC} to capture the breadth of improvement across AL rounds. We report these complementary metrics alongside Dice to ensure uniformly reliable performance across all regions.
We evaluate performance across the same $M$ bins $\{\Omega_m\}_{m=1}^M$ with a per-image, \emph{active-bin}, \emph{unweighted} formulation (naive per-bin Dice is undefined in empty regions). \textbf{SI} measures cross-region disparity at a round and \textbf{PC} the breadth of improvement between rounds; both are reported alongside Dice, which they complement rather than replace.

\begin{figure*}[t]
    \centering
    \includegraphics[width=\textwidth]{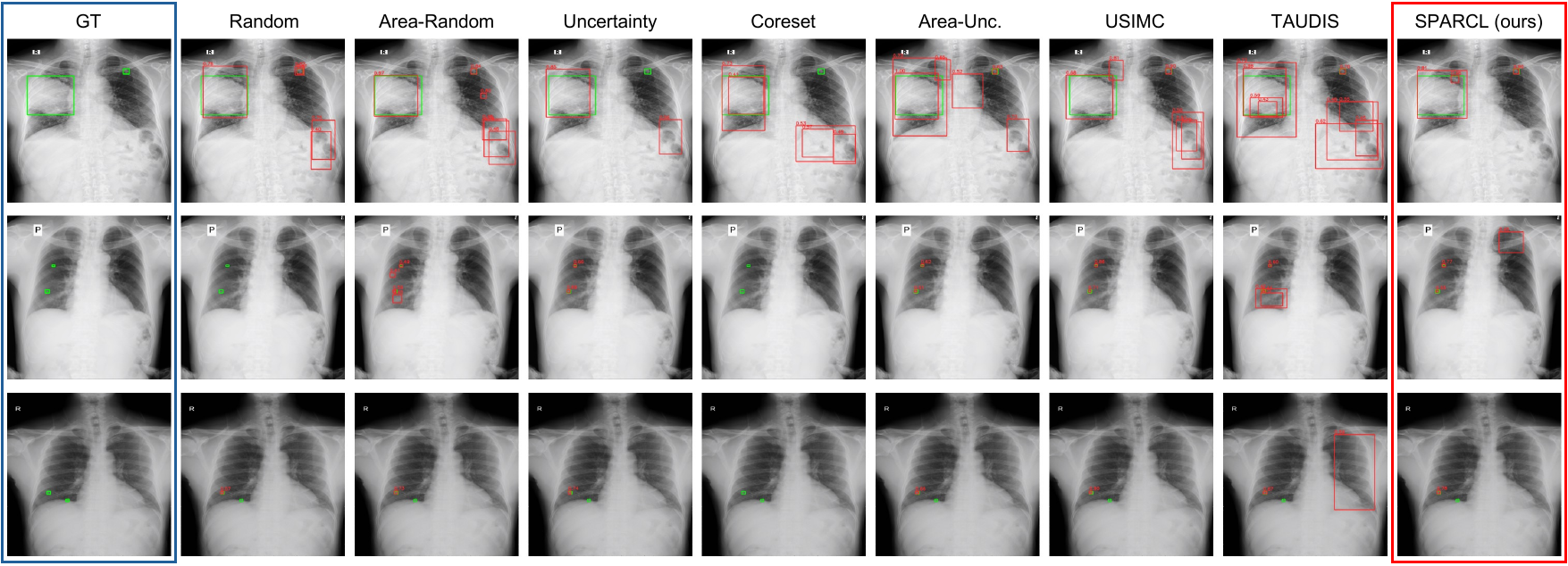}
    \caption{\textbf{Qualitative detection comparison on VinDr-CXR.} Ground-truth boxes (\textcolor{green}{green}) and each method's predictions (\textcolor{red}{red}; confidence $>0.5$, top-5 shown). On the same cases, \emph{SPARCL} localises the dispersed abnormalities more tightly and produces fewer spurious boxes than other baselines.}
    \label{fig:qual_det}
\end{figure*}

\begin{figure}[t]
    \centering
    \includegraphics[width=\columnwidth]{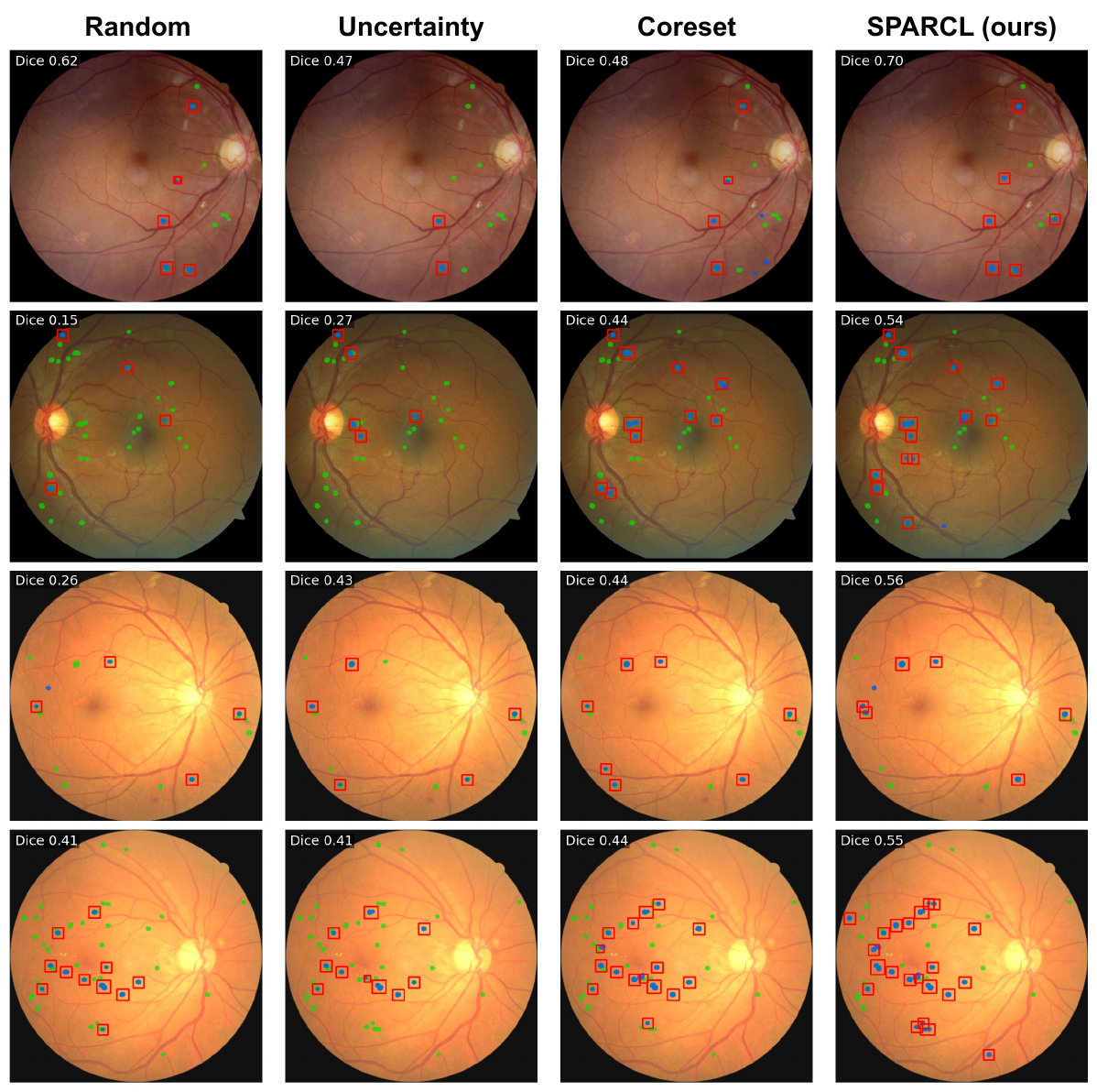}
    \caption{\textbf{Modality generalization to color fundus (DDR, microaneurysms).} Rows: five test cases; columns: acquisition methods. Ground truth in \textcolor{green!60!black}{green}, predictions in \textcolor{blue}{blue}, each true-positive overlap boxed in \textcolor{red}{red}. \emph{SPARCL} recovers more of the tiny scattered lesions (more red boxes, higher per-image Dice) than \textit{Random}, \textit{Uncertainty}, and \textit{Coreset}.}
    \label{fig:qual_fundus}
\end{figure}

\input{tables/table_template_table2_det}

\subsubsection{Per-bin Dice and active bins.} For image $x$ and bin $m$, let $P_m(x)$ and $G_m(x)$ denote the predicted and ground-truth foreground pixel sets, and define the per-bin Dice
\begin{equation}
D_m(x) = \frac{2\,|P_m(x)\cap G_m(x)|}{|P_m(x)|+|G_m(x)|}.
\end{equation}
The \emph{active} set is $\mathcal{A}(x)=\{\,m:\ |P_m(x)|+|G_m(x)|>0\,\}$. This excludes only the truly empty bins ($|G_m|=|P_m|=0$, where Dice is undefined) while \emph{including} false-positive bins ($|G_m|=0,\,|P_m|>0\Rightarrow D_m=0$) and missed-lesion bins ($|G_m|>0,\,|P_m|=0\Rightarrow D_m=0$). Hence scattered false positives in lesion-free regions correctly \emph{penalise} the metric rather than being discarded.

\subsubsection{Spatial Inconsistency (SI).} Let $\overline{D}(x)=\operatorname{mean}_{m\in\mathcal{A}(x)} D_m(x)$. SI is the active-bin coefficient of variation:
\begin{equation}
\mathrm{SI}(x)=
\begin{cases}
\text{n/a}, & |\mathcal{A}(x)|<2,\\[3pt]
0, & |\mathcal{A}(x)|\ge2,\ \overline{D}(x)=0,\\[3pt]
\dfrac{\operatorname{std}_{\,m\in\mathcal{A}(x)} D_m(x)}{\overline{D}(x)}, & |\mathcal{A}(x)|\ge2,\ \overline{D}(x)>0.
\end{cases}
\label{eq:si}
\end{equation}

SI is deliberately \emph{unweighted} to prevent high-mass regions from dominating and obscuring rare-region blind spots. Because $\mathrm{SI}=0$ for both uniformly strong and weak predictions, we always report the triple $\big(\text{overall Dice},\ \overline{D},\ \mathrm{SI}\big).$

\subsubsection{Performance Coverage (PC).} 
To quantify broad improvement between rounds, we average a concave gain ($\psi(z)=\sqrt{z}$) over ground-truth active bins ($\vert{}G_m\vert{}>0$):
\begin{equation}
\begin{aligned}
\mathrm{PC}(t)&=\frac{1}{|\mathcal{A}|}\sum_{m\in\mathcal{A}}\psi\!\big(\Delta D_m(t)\big),\\
\Delta D_m(t)&=\max\!\big(D_m^{\text{post}}(t)-D_m^{\text{pre}}(t),\,0\big),
\end{aligned}
\end{equation}

The concave gain structurally rewards uniform spatial improvements over isolated local gains. Although $\psi$ resembles the acquisition gain, PC evaluates held-out per-bin Dice \emph{improvement}, entirely independent of the pool uncertainty mass optimized by SPARCL.

%% file: tables/alg_1.tex
% \begin{algorithm}[t]
% \footnotesize % 지면 절약을 위해 폰트 크기 조절
% \caption{SPARCL Acquisition Routine}
% \label{alg:sparcl}
% \KwIn{Model $f_{\theta}$; Pool $\mathcal{U}_t$; Budget $b$; Prefilter ratio $r$; Spatial weight $\lambda_{max}$}
% \KwOut{Selected batch $\mathcal{B}_t$}

% % Step 1: Global Stats
% \textbf{1. Global Distribution Analysis:}\;
% \ForEach{$x\in \mathcal{U}_t$}{
%     Estimate entropy map $H(x)$ and spatial histogram $\mathbf{h}(x)$\;
%     $U(x) \leftarrow \text{mean}(H(x))$ \tcp*{Image-level uncertainty}
% }
% $\lambda^{(t)} \leftarrow \lambda_{max} \cdot \frac{\mathcal{H}(\sum_{x\in \mathcal{U}_t} \mathbf{h}(x))}{\log M}$ \tcp*{Dynamic gate via pool entropy}

% % Step 2: Prefiltering
% \textbf{2. Candidate Prefiltering (Optional: for ensure computational scalability):}\;
% $\mathcal{U}' \leftarrow \text{Top } r\text{\% of } \mathcal{U}_t \text{ ranked by } U(x)$\;

% % Step 3: Greedy Selection
% \textbf{3. Adaptive Greedy Selection:}\;
% $\mathcal{S} \leftarrow \varnothing, \mathbf{H}_{\mathcal{S}} \leftarrow \mathbf{0}$\;
% \While{$|\mathcal{S}| < b$}{
%     $x^* \leftarrow \arg\max_{x \in \mathcal{U}' \setminus \mathcal{S}} [U(x) + \lambda^{(t)} \sum_{k=1}^M (\sqrt{H_{\mathcal{S},k} + h_k(x)} - \sqrt{H_{\mathcal{S},k}})]$\;
%     $\mathcal{S} \leftarrow \mathcal{S} \cup \{x^*\}, \mathbf{H}_{\mathcal{S}} \leftarrow \mathbf{H}_{\mathcal{S}} + \mathbf{h}(x^*)$\;
% }
% \Return $\mathcal{B}_t = \mathcal{S}$
% \end{algorithm}

\begin{algorithm2e}[t]
\footnotesize
\DontPrintSemicolon
\LinesNotNumbered
\setlength{\algomargin}{1.6em}
\SetInd{0.4em}{0.6em}
\caption{SPARCL Sample Acquisition (round $t$)}
\label{alg:sparcl}
\KwIn{Segmenter $f_{\theta}$; pool $\mathcal{U}_{t-1}$; labeled set $\mathcal{L}_{t-1}$ with masks $\{y_x\}$; budget $B$; $M$ bins $\{\Omega_m\}$; candidate ratio $r$; gate ceiling $\lambda_{\max}$; concave $\psi(z)=\sqrt{z}$}
\KwOut{Batch $\mathcal{B}_t$}
% \tcp{1. Per-image uncertainty and its spatial distribution}
\ForEach{$x\in \mathcal{U}_{t-1}$}{
    $H(i)\leftarrow$ pixel entropy of $f_\theta(x)$;\ \ $U(x)\leftarrow\tfrac{1}{|\Omega|}\sum_{i} H(i)$;\ \ $q_m(x)\leftarrow\sum_{i\in\Omega_m} H(i)\big/\sum_{i\in\Omega} H(i)$\;
}
% \tcp{2. Candidate set (uncertainty floor)}
$\mathcal{U}'_t\leftarrow$ top-$rB$ of $\mathcal{U}_{t-1}$ by $U(x)$;\ \ $U(x)\leftarrow\big(U(x)-\min_{\mathcal{U}'_t}U\big)\big/\big(\max_{\mathcal{U}'_t}U-\min_{\mathcal{U}'_t}U\big)$\;
% \tcp{3. Positional-diversity gate (closes only if uncertainty sits in a few bins for all candidates)}
$\bar q_m\leftarrow\tfrac{1}{|\mathcal{U}'_t|}\sum_{x\in\mathcal{U}'_t} q_m(x)$;\ \ $\lambda^{(t)}\leftarrow\lambda_{\max}\,\mathcal{H}(\bar{\mathbf q})/\log M$\;
% \tcp{4. Cumulative coverage state from already-labeled lesions}
$g_m(x)\leftarrow|\{i\in\Omega_m: y_x(i)=1\}|\big/|\{i: y_x(i)=1\}|$ for $x\in\mathcal{L}^{+}_{t-1}$;\ \ $H_{\mathcal{S},m}\leftarrow\tfrac{B}{|\mathcal{L}^{+}_{t-1}|}\sum_{x\in\mathcal{L}^{+}_{t-1}} g_m(x)$;\ \ $\mathcal{S}\leftarrow\varnothing$\;
% \tcp{5. Greedy batch construction}
\While{$|\mathcal{S}| < B$}{
    $x^{\star}\leftarrow\arg\max_{x\in\mathcal{U}'_t\setminus\mathcal{S}}\ U(x)+\lambda^{(t)}\sum_{m}\big[\psi(H_{\mathcal{S},m}\!+\!q_m(x))-\psi(H_{\mathcal{S},m})\big]$\;
    $\mathcal{S}\leftarrow\mathcal{S}\cup\{x^{\star}\}$;\ \ $H_{\mathcal{S},m}\mathrel{+}= q_m(x^{\star})$\;
}
\Return $\mathcal{B}_t=\mathcal{S}$ \tcp*{greedy on monotone submodular $A$: $(1{-}1/e)$-optimal (Cor.~\ref{cor:greedy})}
\tcp{Detection: $q_m(x)\!\propto\!\sum_{o\in\mathcal{B}_x} s_o\mathcal{H}(s_o)\,\mathrm{IoU}(o,\Omega_m)$, $U(x)\!\leftarrow\!\sum_{o} s_o\mathcal{H}(s_o)$ (box score $s_o$), $g_m(x)\!\propto$ ground-truth box area in $\Omega_m$; steps 2--5 unchanged}
\end{algorithm2e}

%% file: tables/table_template_table1_seg.tex
\begin{table*}[t]
\centering
\caption{
Label efficiency on segmentation benchmarks, mean $\pm$ std over five shared random seeds. Dice@20\% is the per-image Dice ($\times$100) at a 20\% annotation budget, Random held-out split; AULC is the normalized area under the Dice-vs-budget curve. Budget@95\% is the minimum budget to reach 95\% of full-data Dice (in-house 46.8~$\pm$~1.7, SIIM-ACR 61.8~$\pm$~0.4 at 100\% budget). Bold = best in column, Underline = second best in column.
}
\label{tab:segmentation_sota}
\scriptsize
\setlength{\tabcolsep}{10pt}
\renewcommand{\arraystretch}{1.15}
\resizebox{0.8\textwidth}{!}{
\begin{tabular}{lccc|ccc}
\toprule
& \multicolumn{3}{c|}{\textbf{In-house Calcification}}
& \multicolumn{3}{c}{\textbf{SIIM-ACR}} \\
\cmidrule(lr){2-4}
\cmidrule(lr){5-7}
\textbf{Method}
& \textbf{Dice@20\%\,$\uparrow$} & \textbf{AULC\,$\uparrow$} & \textbf{Budget@95\%\,$\downarrow$}
& \textbf{Dice@20\%\,$\uparrow$} & \textbf{AULC\,$\uparrow$} & \textbf{Budget@95\%\,$\downarrow$} \\
\midrule
Random
& 37.3 $\pm$ 1.2 & 42.7 $\pm$ 1.3 & 29.7 $\pm$ 3.3\%
& 52.9 $\pm$ 2.2 & 56.9 $\pm$ 1.5 & 25.5 $\pm$ 3.2\% \\

Area-Random
& 38.3 $\pm$ 0.4 & 43.0 $\pm$ 1.2 & 29.9 $\pm$ 1.6\%
& \underline{54.1 $\pm$ 2.1} & \underline{59.3 $\pm$ 2.3} & 24.7 $\pm$ 2.1\% \\

Uncertainty
& 41.2 $\pm$ 1.1 & \underline{43.7 $\pm$ 0.5} & \underline{29.3 $\pm$ 0.9\%}
& 53.6 $\pm$ 1.8 & 58.9 $\pm$ 1.6 & \underline{22.4 $\pm$ 1.1\%} \\

Area-Uncertainty
& 41.5 $\pm$ 0.2 & 43.4 $\pm$ 0.8 & 29.5 $\pm$ 0.7\%
& 53.8 $\pm$ 0.5 & 58.1 $\pm$ 2.2 & 24.2 $\pm$ 1.6\% \\

Coreset
& 40.9 $\pm$ 2.2 & 43.4 $\pm$ 0.9 & 29.9 $\pm$ 2.2\%
& 52.7 $\pm$ 1.0 & 57.5 $\pm$ 0.7 & 31.3 $\pm$ 2.8\% \\

Diversity
& 38.4 $\pm$ 2.1 & 42.3 $\pm$ 1.1 & 57.2 $\pm$ 2.9\%
& 51.1 $\pm$ 1.4 & 56.8 $\pm$ 2.1 & 29.5 $\pm$ 3.1\% \\

USIMC
& 41.0 $\pm$ 1.8 & 42.6 $\pm$ 0.7 & 60.1 $\pm$ 1.8\%
& 52.6 $\pm$ 1.7 & 56.9 $\pm$ 1.5 & 25.2 $\pm$ 1.2\% \\

TAUDIS
& 37.6 $\pm$ 1.1 & 42.7 $\pm$ 0.6 & 41.4 $\pm$ 1.6\%
& 51.7 $\pm$ 2.1 & 55.8 $\pm$ 1.4 & 34.6 $\pm$ 1.9\% \\

Loss-Prediction
& \underline{42.5 $\pm$ 1.2} & 42.9 $\pm$ 0.4 & 36.4 $\pm$ 2.1\%
& 51.3 $\pm$ 2.0 & 55.9 $\pm$ 1.8 & 28.1 $\pm$ 1.8\% \\

PAAL
& 35.4 $\pm$ 1.9 & 41.9 $\pm$ 0.3 & 67.2 $\pm$ 2.4\%
& 49.4 $\pm$ 2.2 & 55.0 $\pm$ 2.2 & 26.8 $\pm$ 2.2\% \\

\rowcolor{Gray} \textbf{SPARCL (Ours)}
& \textbf{43.4 $\pm$ 1.3} & \textbf{45.1 $\pm$ 0.8} & \textbf{24.4 $\pm$ 1.4\%}
& \textbf{55.1 $\pm$ 1.4} & \textbf{60.6 $\pm$ 1.7} & \textbf{21.2 $\pm$ 2.1\%} \\
\bottomrule
\end{tabular}
}
\end{table*}

%% file: tables/table_template_table3_spatial_reliability_metric.tex
\begin{table*}[t]
\centering
% \caption{
% Spatial reliability at a fixed 20\% annotation budget on the in-house
% calcification dataset. Global/Mean Bin Dice and Anat-Zone Worst are pooled micro-Dice over the 7 lung zones; Dead-Bin\% is the share of GT-bearing $7{\times}7$ bins with zero pooled overlap; SI is the revised per-image spatial-inequality index ($\downarrow$) and PC the performance coverage ($\uparrow$). Results are mean $\pm$ std over available five shared seeds.}
\caption{
Spatial reliability at a fixed 20\% annotation budget on the in-house calcification dataset. Global Dice is the pooled micro-Dice ($\times$100) over the test set; Mean Bin Dice is the mean pooled Dice over the $7{\times}7$ grid bins, Random held-out split; Anat-Zone Worst is the lowest pooled Dice over the seven lung zones (Fig.~\ref{fig:anatomy_bins}); Dead-Bin\% is the share of GT-bearing $7{\times}7$ bins with zero pooled overlap. SI is the Spatial Inconsistency of Eq.~\eqref{eq:si} ($\downarrow$) and PC the Performance Coverage ($\uparrow$). Mean $\pm$ std over five shared seeds; ties are broken by the lower std. Bold = best in column, Underline = second best.
}
\label{tab:spatial_reliability}
\scriptsize
\setlength{\tabcolsep}{10pt}
\renewcommand{\arraystretch}{1.15}
\begin{tabular}{lcccccc}
\toprule
\textbf{Method}
& \textbf{Global Dice\,$\uparrow$}
& \textbf{Mean Bin Dice\,$\uparrow$}
& \textbf{Dead-Bin \%\,$\downarrow$}
& \textbf{Anat-Zone Worst\,$\uparrow$}
& \textbf{SI\,$\downarrow$}
& \textbf{PC\,$\uparrow$} \\
\midrule

Random
& 28.8 $\pm$ 3.1 & 22.0 $\pm$ 2.1
& 23.0 $\pm$ 3.3 & 11.3 $\pm$ 3.3
& 0.931 $\pm$ 0.181 & 0.138 $\pm$ 0.039 \\

Area-Random
& 29.2 $\pm$ 2.8 & 27.3 $\pm$ 3.1
& 17.6 $\pm$ 2.0 & 10.6 $\pm$ 2.7
& 0.847 $\pm$ 0.109 & 0.137 $\pm$ 0.042 \\

Uncertainty
& \underline{39.8 $\pm$ 1.1} & 31.0 $\pm$ 2.0
& 16.0 $\pm$ 1.2 & \underline{16.6 $\pm$ 2.9}
& 0.750 $\pm$ 0.045 & 0.111 $\pm$ 0.035 \\

Area-Uncertainty
& 38.1 $\pm$ 1.2 & \underline{33.7 $\pm$ 2.4}
& 12.0 $\pm$ 1.5 & 16.4 $\pm$ 3.3
& 0.879 $\pm$ 0.119 & 0.137 $\pm$ 0.022 \\

Coreset
& 36.9 $\pm$ 1.3 & 29.4 $\pm$ 2.3
& 16.1 $\pm$ 2.5 & 14.5 $\pm$ 2.9
& 0.751 $\pm$ 0.023 & 0.098 $\pm$ 0.009 \\

Diversity
& 31.2 $\pm$ 3.1 & 28.7 $\pm$ 3.4
& 18.8 $\pm$ 3.2 & 15.7 $\pm$ 4.6
& \underline{0.655 $\pm$ 0.071} & 0.133 $\pm$ 0.042 \\

USIMC
& 38.9 $\pm$ 2.4 & 29.0 $\pm$ 3.3
& 19.4 $\pm$ 1.9 & 15.0 $\pm$ 3.8
& 0.796 $\pm$ 0.108 & 0.128 $\pm$ 0.050 \\

TAUDIS
& 27.5 $\pm$ 3.6 & 25.0 $\pm$ 3.0
& \underline{10.8 $\pm$ 2.9} & 12.5 $\pm$ 2.5
& 0.735 $\pm$ 0.082 & 0.128 $\pm$ 0.034 \\

Loss-Prediction
& 31.2 $\pm$ 2.9 & 24.5 $\pm$ 4.0
& 23.0 $\pm$ 2.7 & 14.5 $\pm$ 4.2
& 0.942 $\pm$ 0.137 & \underline{0.146 $\pm$ 0.013} \\

PAAL
& 30.7 $\pm$ 1.2 & 24.3 $\pm$ 2.5
& 20.0 $\pm$ 1.1 & 6.6 $\pm$ 3.7
& 0.950 $\pm$ 0.172 & 0.120 $\pm$ 0.047 \\

\rowcolor{Gray} \textbf{SPARCL (Ours)}
& \textbf{40.5 $\pm$ 1.1}
& \textbf{36.1 $\pm$ 2.3}
& \textbf{9.3 $\pm$ 1.2}
& \textbf{19.2 $\pm$ 3.0}
& \textbf{0.571 $\pm$ 0.038}
& \textbf{0.148 $\pm$ 0.009} \\
\bottomrule
\end{tabular}

\end{table*}

%% file: tables/table_template_table2_det.tex
\begin{table*}[t]
\centering
\caption{
Label efficiency and native object-detection performance on public detection benchmarks, mean $\pm$ std over five shared random seeds. mAP$^{\dagger}$, AP$_{50}$$^{\dagger}$, and FROC$^{\dagger}$ are at a fixed $20\%$ budget; AULC is the normalized area under the mAP-vs-budget curve; Budget@95\% is the minimum budget to reach $95\%$ of full-data mAP (VinDr-CXR 9.5~$\pm$~0.1 mAP / 25.8~$\pm$~0.2 AP$_{50}$; ChestX-Det10 13.2~$\pm$~0.4 mAP / 36.5~$\pm$~0.3 AP$_{50}$ at 100\% budget). Bold = best in column, Underline = second best in column.
}
\label{tab:detection_sota}
\scriptsize
\setlength{\tabcolsep}{5pt}
\renewcommand{\arraystretch}{1.15}
\begin{tabular}{lccccc|ccccc}
\toprule
& \multicolumn{5}{c|}{\textbf{VinDr-CXR}}
& \multicolumn{5}{c}{\textbf{ChestX-Det10}} \\
\cmidrule(lr){2-6}
\cmidrule(lr){7-11}
\textbf{Method}
& \textbf{mAP$^{\dagger}$\,$\uparrow$} & \textbf{AP$_{50}$$^{\dagger}$\,$\uparrow$} & \textbf{FROC$^{\dagger}$\,$\uparrow$} & \textbf{AULC\,$\uparrow$} & \textbf{Budget@95\%\,$\downarrow$}
& \textbf{mAP$^{\dagger}$\,$\uparrow$} & \textbf{AP$_{50}$$^{\dagger}$\,$\uparrow$} & \textbf{FROC$^{\dagger}$\,$\uparrow$} & \textbf{AULC\,$\uparrow$} & \textbf{Budget@95\%\,$\downarrow$} \\
\midrule

Random
& 7.1 $\pm$ 0.3 & 19.9 $\pm$ 0.6 & 25.0 $\pm$ 0.5 & 9.3 $\pm$ 0.3 & 49 $\pm$ 8\%
& 8.5 $\pm$ 0.4 & 25.8 $\pm$ 0.5 & \underline{26.5 $\pm$ 0.3} & 10.5 $\pm$ 0.2 & 75 $\pm$ 4\% \\

Area-Random
& \underline{7.3 $\pm$ 0.2} & \underline{20.7 $\pm$ 0.2} & \underline{25.7 $\pm$ 0.7} & \underline{9.6 $\pm$ 0.1} & \underline{42 $\pm$ 4\%}
& 8.2 $\pm$ 0.2 & \underline{26.1 $\pm$ 0.2} & 26.4 $\pm$ 0.2 & 9.5 $\pm$ 0.1 & 82 $\pm$ 1\% \\

Uncertainty
& 6.7 $\pm$ 0.4 & 17.5 $\pm$ 1.1 & 22.9 $\pm$ 0.5 & 8.6 $\pm$ 0.1 & 48 $\pm$ 2\%
& 7.6 $\pm$ 0.3 & 24.4 $\pm$ 0.6 & 25.4 $\pm$ 0.5 & 9.9 $\pm$ 0.1 & 89 $\pm$ 4\% \\

Area-Uncertainty
& 7.1 $\pm$ 0.1 & 20.4 $\pm$ 1.0 & 24.8 $\pm$ 0.7 & 9.3 $\pm$ 0.3 & 47 $\pm$ 6\%
& 7.3 $\pm$ 0.3 & 23.3 $\pm$ 0.6 & 24.2 $\pm$ 0.5 & 10.0 $\pm$ 0.2 & 80 $\pm$ 2\% \\

Coreset
& 6.9 $\pm$ 0.3 & 19.7 $\pm$ 0.7 & 25.5 $\pm$ 0.9 & 9.5 $\pm$ 0.3 & 45 $\pm$ 4\%
& \underline{8.6 $\pm$ 0.2} & 25.7 $\pm$ 0.3 & 26.2 $\pm$ 0.2 & \underline{10.6 $\pm$ 0.1} & \underline{73 $\pm$ 3\%} \\

Diversity
& 6.9 $\pm$ 0.2 & 19.1 $\pm$ 0.9 & 25.1 $\pm$ 1.0 & 9.1 $\pm$ 0.1 & 43 $\pm$ 2\%
& 8.5 $\pm$ 0.1 & 26.0 $\pm$ 0.9 & 25.7 $\pm$ 0.7 & 9.6 $\pm$ 0.2 & 81 $\pm$ 1\% \\

TAUDIS
& 6.4 $\pm$ 0.2 & 18.4 $\pm$ 0.7 & 23.7 $\pm$ 0.9 & 9.3 $\pm$ 0.3 & 46 $\pm$ 2\%
& 8.5 $\pm$ 0.3 & 25.2 $\pm$ 0.2 & 25.9 $\pm$ 0.3 & 10.4 $\pm$ 0.1 & 79 $\pm$ 2\% \\

USIMC
& 6.4 $\pm$ 0.2 & 18.4 $\pm$ 1.3 & 24.3 $\pm$ 1.5 & 8.8 $\pm$ 0.2 & 58 $\pm$ 3\%
& 7.7 $\pm$ 0.3 & 24.3 $\pm$ 0.6 & 25.1 $\pm$ 0.3 & 10.0 $\pm$ 0.2 & 81 $\pm$ 3\% \\

\rowcolor{Gray} \textbf{SPARCL (Ours)}
& \textbf{7.7 $\pm$ 0.1} & \textbf{21.7 $\pm$ 0.8} & \textbf{27.1 $\pm$ 0.8} & \textbf{9.8 $\pm$ 0.2} & \textbf{40 $\pm$ 3\%}
& \textbf{8.9 $\pm$ 0.1} & \textbf{26.2 $\pm$ 0.5} & \textbf{26.6 $\pm$ 0.4} & \textbf{10.7 $\pm$ 0.1} & \textbf{71 $\pm$ 2\%} \\
\bottomrule
\end{tabular}

\vspace{2pt}
\end{table*}

%% file: contents/3_experiments.tex
\section{Experiments}
\label{sec:experiments}

% \begin{figure}[h!]
%     \centering
%     \includegraphics[width=\columnwidth]{figures/qual_detection_vindr.pdf}
%     \caption{\textbf{Qualitative detection comparison on VinDr-CXR.} Ground-truth boxes (\textcolor{green}{green}) and each method's predictions (\textcolor{red}{red, dashed}; confidence $>0.5$, top-5 shown). On the same cases, \emph{SPARCL} localises the dispersed abnormalities more tightly and produces fewer spurious boxes than \textit{Random}, \textit{Coreset}, and \textit{Uncertainty}.}
%     \label{fig:qual_det}
% \end{figure}

\subsection{Datasets and Preprocessing}
To validate adaptability across distinct spatial regimes, we evaluate \emph{SPARCL} on four datasets:
\textbf{(1) In-house Calcification (Seg., Mixed):} 2.8k private CXRs with two expert radiologists annotations for \textit{Calcified Nodules} (900, diffuse) and \textit{Aortic Calcification} (900, focal).
\textbf{(2) VinDr-CXR~\cite{nguyen2022vindr} (Det., Diffuse):} 15k CXRs with bounding boxes for \textit{Pulmonary Nodules} and other scattered abnormalities.
\textbf{(3) SIIM-ACR~\cite{filice2020crowdsourcing} (Seg., Focal):} 10.7k CXRs for \textit{Pneumothorax} targeting the pleural boundary.
\textbf{(4) ChestX-Det10~\cite{liu2020chestx} (Det., Mixed):} A 3.5k instance-level subset of NIH ChestX-ray14 covering 10 disease categories. 

To probe generalization beyond chest radiography, we additionally use \textbf{(5) DDR~\cite{li2019diagnostic} (Seg., Diffuse):} color fundus photographs for \textit{microaneurysm} segmentation, a strongly diffuse regime with tiny, scattered retinal lesions. DDR is used only for the cross-modality generalization study (Table~\ref{tab:fundus_generalization}), not for the main CXR comparisons.

% \subsection{Implementation Details}
% \subsubsection{Training \& AL Protocol.} All experiments are implemented in PyTorch using single NVIDIA V100 (32GB) GPUs.
% For segmentation, we employ a U-Net~\cite{ronneberger2015u} (EfficientNet-B4 encoder, scSE decoder) optimized via AdamW (LR=$3\text{e-}4$) with Dice+BCE loss.
% For detection, we use Faster R-CNN (ResNet-50) with focal and smooth L1 losses.
% Models are trained for up to 150 epochs (batch size 16) with early stopping (patience 10, $\delta=0.001$).
% Default AL parameters are $K=L=7$ ($7\times7$ grid) and $\lambda_{\max}=1.0$.
% In each AL round, we expand the labeled pool by acquiring an additional 1\% of the total dataset. To isolate the effect of the acquisition rule, all strategies share identical network architectures, optimizers, data augmentations, training schedules, and early-stopping criteria; only the sample-selection rule differs. For each random seed, every method starts from the \emph{same} initial labeled set and unlabeled pool and acquires the \emph{same} number of samples at every AL round, so any performance difference is attributable to acquisition alone. All methods are re-evaluated with the same set of shared random seeds, and unless otherwise noted results are reported as mean $\pm$ standard deviation over five shared seeds. Splits are patient-level to avoid leakage between the labeled pool and the held-out evaluation set.

\subsection{Implementation Details}
\subsubsection{Training and AL Protocol}
All experiments are implemented in PyTorch on a single NVIDIA V100 (32\,GB) GPU. To isolate the effect of the acquisition rule, all AL strategies share identical architectures (U-Net~\cite{ronneberger2015u} with an EfficientNet-B4 encoder and scSE decoder trained with a Dice+BCE loss for segmentation; Faster R-CNN with a ResNet-50 backbone trained with focal and smooth-$L_1$ losses for detection), optimizers (AdamW, LR=$3\text{e-}4$), augmentations, 150 epochs (batch size 16), and early-stopping (patience 10, $\delta=0.001$) criteria. In each AL round, we expand the labeled pool by 1\% of the total dataset. Default \emph{SPARCL} parameters are a $7 \times 7$ grid ($K=L=7$), $\lambda_{\max}=1.0$, and candidate ratio $r=5$; all three were selected on the validation split of the in-house cohort. Every method starts from the same initial labeled set and acquires the exact same number of samples per round. Results are reported as mean $\pm$ standard deviation over five shared random seeds using patient-level splits to avoid data leakage.

% \subsubsection{Evaluation Split Protocols.} We support three test split modes:
% (1) \textit{Random} (default): a patient-level random split that ignores lesion location, so the test set follows the natural spatial prevalence of the pathology.
% (2) \textit{Spatial-Equal}: the test set is assembled so that each of the $K\times L$ bins contributes an equal number of lesion-bearing images ($N_{\text{test}}/(K\times L)$ per bin), giving every region equal weight regardless of prevalence; ground-truth masks are used only to assign test images to bins.
% (3) \textit{Spatial-Dynamic}: Proportional to each bin's lesion prevalence, $N_{\text{test}}\cdot|\Omega_m|/\sum_j|\Omega_j|$, so the evaluation distribution mirrors the natural spatial distribution of the pathology.

\subsubsection{Evaluation Split Protocols.} The held-out set is drawn in one of three modes, using ground-truth masks only to assign images to bins: \textit{Random} (patient-level, ignores lesion location, follows natural spatial prevalence); \textit{Spatial-Equal} (each of the $K\times L$ bins contributes $N_{\text{test}}/(K\times L)$ lesion-bearing images, so every region is weighted equally regardless of prevalence); and \textit{Spatial-Dynamic} (bin quotas $N_{\text{test}}\cdot|\Omega_m|/\sum_j|\Omega_j|$ proportional to lesion prevalence).

\subsubsection{Generalization to Object Detection.}
To adapt \emph{SPARCL} for detection, we reformulate the spatial histogram using predicted boxes $\mathcal{B}_x$: $h_m(x) = \sum_{b \in \mathcal{B}_x} s_b \mathcal{H}(s_b) \cdot \text{IoU}(b, \text{bin}_m)$, where $s_b$ is the confidence score and $\mathcal{H}$ is binary entropy.
The image-level uncertainty is correspondingly the aggregate box uncertainty $U(x) = \sum_{b \in \mathcal{B}_x} s_b \mathcal{H}(s_b)$, which replaces the mean pixel entropy used for segmentation.
This $s_b \mathcal{H}(s_b)$ weighting filters background noise by prioritizing object-like (high $s_b$) yet uncertain regions. The histogram is normalized to $\mathbf{q}(x)$ exactly as for segmentation; for the cumulative coverage state, $g_m(x)$ is the fraction of ground-truth box area of a labeled image that falls in bin $m$. The candidate set, gate, coverage state, and greedy step are otherwise unchanged.

\subsection{Baselines}
We compare \emph{\textbf{SPARCL}} against 10 strategies across four categories. \textbf{Random:} \textbf{(1) \textit{Random}} (uniform pool selection) and \textbf{(2) \textit{Area-Random}} (spatially-stratified).
\textbf{Uncertainty:} \textbf{(3) \textit{Uncertainty}}~\cite{yang2017suggestive} (standard pixel/box entropy) and \textbf{(4) \textit{Area-Uncertainty}} (entropy with spatial stratification).
\textbf{Diversity:} \textbf{(5) \textit{Coreset}}~\cite{sener2017active} (global k-center feature coverage over the entire labeled set) and \textbf{(6) \textit{Diversity}} (intra-batch farthest-point feature sampling).
\textbf{Hybrid / learned:} \textbf{(7) \textit{TAUDIS}}~\cite{yu2023two} (uncertainty filtering + diversity), \textbf{(8) \textit{USIMC}}~\cite{follmer2024active} (uncertainty-aware submodular mutual information), \textbf{(9) \textit{Loss-Prediction}}~\cite{yoo2019learning} (a module predicts each sample's target loss), and \textbf{(10) \textit{PAAL}}~\cite{shi2024predictive} (an auxiliary network predicts each sample's segmentation accuracy and the lowest-predicted-accuracy samples are queried). Note that \emph{PAAL} and \emph{Loss-Prediction} are designed specifically for segmentation and are therefore applied only to the segmentation benchmarks.

%% file: contents/4_results.tex
\input{tables/table_fundus}

% \begin{figure}[t]
%     \centering
%     \includegraphics[width=\columnwidth]{figures/qual_inhouse_advantage.pdf}
%     \caption{\textbf{Qualitative segmentation on the in-house calcified-nodule cohort (diffuse).} Ground-truth contour (green) and predicted mask (red overlay) with per-image Dice. On these scattered calcifications, \textit{Random} and \textit{Uncertainty} often miss the lesion (Dice~$\approx0$) while \emph{SPARCL} recovers it; removing the entropy gate (\emph{SPARCL w/o gate}) degrades the prediction, showing the gate's importance in the diffuse regime.}
%     \label{fig:qual_inhouse}
% \end{figure}
\begin{figure*}[t]
    \centering
    \includegraphics[width=\textwidth]{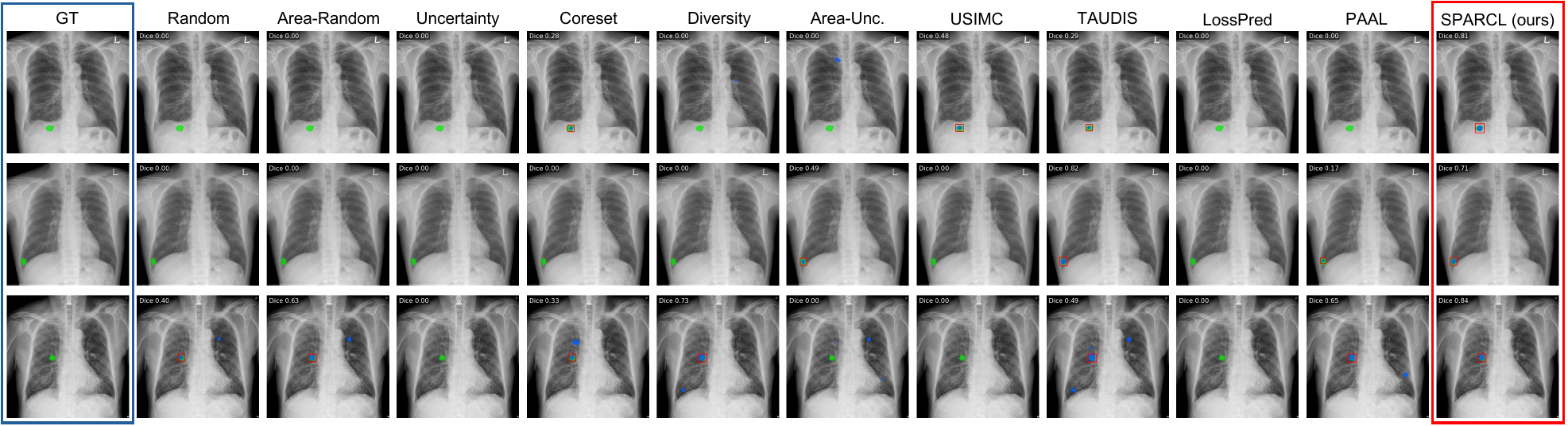}
    \caption{\textbf{Qualitative segmentation on the in-house calcified-nodule cohort.} Ground-truth contours in \textcolor{green!60!black}{green}, predictions in \textcolor{blue}{blue}, and true-positive overlaps boxed in \textcolor{red}{red}, shown with per-image Dice. Baselines often miss these scattered calcifications (Dice~$\approx0$) while \emph{SPARCL} recovers them.}
    \label{fig:qual_inhouse}
\end{figure*}

\subsection{Quantitative Analysis}

% Fig.~\ref{fig:alc_budget_curves} shows the learning dynamics across annotation budgets (10--100\%). \emph{SPARCL} demonstrates label efficiency and spatial reliability, effectively mitigating the trade-off between aggregate performance and anatomical fairness.

\subsubsection{Label Efficiency and Spatial Reliability}
At a restricted 20\% annotation budget, \emph{SPARCL} consistently outperforms all baselines. On the in-house dataset (Table~\ref{tab:segmentation_sota}), it achieves a Dice score of 43.4\%, surpassing \textit{Random} (37.3\%) and \textit{Uncertainty} sampling (41.2\%). Moreover, \textit{Uncertainty} sampling generates false positives in local blind spots, exhibiting limited spatial stability (Fig.~\ref{fig:selection_map}b). Regarding spatial safety (Table~\ref{tab:spatial_reliability}), \emph{SPARCL} maintains the lowest Spatial Inconsistency ($SI=0.571$) and the highest Performance Coverage ($PC=0.148$) at a 20\% budget. By consistently yielding the highest Anat-Zone Worst Dice (19.2) and the lowest Dead-Bin rate (9.3\%), \emph{SPARCL} prevents local overfitting and substantially reduces anatomical blind spots (Dead-Bin 9.3\% vs.\ 16.0\% for \textit{Uncertainty}).

% Dead-Bin % (↓ 낮을수록 좋음) : GT가 있는 7×7 칸 중 완전히 놓친(pooled Dice = 0) 칸의 비율 =  "병변 있는 영역 중 몇 %를 아예 0으로 놓쳤나." - 얼마나 많은 영역을 통째로 놓쳤나
% Anat-Zone Worst (↑ 높을수록 좋음): 7개 폐 해부학 zone(좌/우 × apex/mid/base + central) 중 가장 성능 나쁜 zone의 pooled Dice. = "최악의 해부학 영역도 얼마나 하나."

\subsubsection{Generalization Across Diverse Patterns}
Table~\ref{tab:segmentation_sota} and Table~\ref{tab:detection_sota} summarize efficiency across diverse public datasets. \emph{SPARCL} requires less data to reach 95\% of fully supervised performance (Budget@95\%). It reaches this milestone at only a 24\% budget on the in-house cohort and 21\% on the focal SIIM-ACR dataset. This extends to the highly dispersed VinDr-CXR (40\% budget) and instance-level ChestX-Det10 (71\% budget) detection tasks, i.e., across lesion distributions and task types.

\subsubsection{Object Detection.}
% Because \emph{SPARCL}'s acquisition relies only on a spatial uncertainty histogram, it transfers naturally to object detection by reformulating the per-bin mass from predicted boxes (Algorithm~\ref{alg:sparcl}). Evaluating a Faster R-CNN on VinDr-CXR and ChestX-Det10 (Table~\ref{tab:detection_sota}), \emph{SPARCL} attains the best mAP (7.7\% and 8.9\%), AP$_{50}$ (21.7\% and 26.2\%), and FROC (27.1\% and 26.6\%) at a 20\% budget. Fig.~\ref{fig:qual_det} shows it localizes dispersed findings more tightly and with fewer false positives than baselines, confirming that entropy-gated spatial coverage directly benefits detection.

With the box-based histogram of Sec.~\ref{sec:experiments}, a Faster R-CNN on VinDr-CXR and ChestX-Det10 (Table~\ref{tab:detection_sota}) attains the best mAP (7.7\% and 8.9\%), AP$_{50}$ (21.7\% and 26.2\%), and FROC (27.1\% and 26.6\%) at a 20\% budget under \emph{SPARCL}. Fig.~\ref{fig:qual_det} shows it localizes dispersed findings more tightly and with fewer false positives than baselines.

\subsubsection{Modality Generalization to Fundus.}
Beyond CXR, the same pipeline transfers to color fundus photography (DDR)~\cite{li2019diagnostic} for microaneurysm segmentation, a strongly \emph{diffuse} regime with tiny retinal lesions. Table~\ref{tab:fundus_generalization} and Fig.~\ref{fig:qual_fundus} show that \emph{SPARCL} recovers more scattered lesions than baselines, and reaches highest dice score, indicating cross-modality adaptability.

\subsection{Qualitative Analysis}
% Fig.~\ref{fig:selection_map} visualizes sample selection.
% Standard \textit{Uncertainty} sampling (top) exhibits severe \textit{spatial redundancy}, clustering queries around high-contrast edges while leaving large anatomical areas undersampled.
% This creates ``\emph{anatomical blind spots}'', causing the model to generate severe false positives (red circles) due to a lack of learned local features.
% Conversely, \emph{SPARCL} (bottom) utilizes entropy-gating to enforce a spatially balanced acquisition.
% By securing coverage across the entire lung field, our method effectively suppresses false positives and successfully segments the target lesion, confirming that spatial fairness directly translates to clinical reliability. Fig.~\ref{fig:qual_inhouse} demonstrates this on the in-house calcified-nodule cohort, where small, scattered lesions strictly require spatial coverage. While baselines frequently miss these diffuse nodules (Dice~$\approx0$), \emph{SPARCL} successfully recovers them. 
Fig.~\ref{fig:selection_map} visualizes acquisition: \textit{Uncertainty} sampling clusters queries around high-contrast edges and leaves large anatomical areas undersampled, producing ``\emph{anatomical blind spots}'' where the model later emits false positives (red circles), whereas \emph{SPARCL}'s cumulative coverage spreads queries over the lung field and the corresponding false positives disappear. Fig.~\ref{fig:qual_inhouse} shows the same effect on the in-house calcified-nodule cohort, where baselines frequently miss the small, scattered lesions (Dice~$\approx0$) that \emph{SPARCL} recovers.

\subsection{Ablation Study}

% \subsubsection{Entropy-Gate Ablation}
% % Table~\ref{tab:gate_ablation} isolates the contribution of the adaptive gate $\lambda^{(t)}$ on SIIM-ACR (5 seeds). Structurally fixing the gate ($\lambda=0.5$) or removing it completely ($\lambda\equiv\lambda_{\max}$) reverts \emph{SPARCL} to static region-balancing, which inflates the Budget@95\% from 21\% to 24\% and 33\%, respectively. Furthermore, removing the gate severely degrades spatial fairness, worsening $SI$ (from 0.553 to 0.629) and lowering $PC$ (from 0.176 to 0.167). The Budget@95\% gap (21\% vs.\ 33\%) exceeds the seed spread, while $SI$ and $PC$ move in the same direction but within one standard deviation; dynamic, entropy-aware gating is therefore what delivers label efficiency and spatial reliability together.
% Table~\ref{tab:gate_ablation} isolates the contribution of the adaptive gate $\lambda^{(t)}$ on SIIM-ACR (5 seeds). Structurally fixing the gate ($\lambda=0.5$) or removing it completely ($\lambda\equiv\lambda_{\max}$) reverts \emph{SPARCL} to static region-balancing, which inflates the Budget@95\% from 21\% to 24\% and 33\%, respectively. Furthermore, removing the gate severely degrades spatial fairness, worsening $SI$ (from 0.553 to 0.629) and lowering $PC$ (from 0.176 to 0.167). The Budget@95\% gap (21\% vs.\ 33\%) exceeds the seed spread, while $SI$ and $PC$ move in the same direction but within one standard deviation; dynamic, entropy-aware gating is therefore what delivers label efficiency and spatial reliability together.

\subsubsection{Component Ablation}
Table~\ref{tab:ablation_components} isolates the three components of \emph{SPARCL} on the in-house cohort. (A)~Resetting $H_0$ every round, i.e.\ within-batch coverage only, costs about two points of Budget@95\% (24.4\%$\rightarrow$26.5\%) and 2.2 points of Dice@20\%, and dropping $U(x)$ costs over four (28.7\%, Dice 37.1): coverage helps only when it redirects the budget \emph{among} uncertain images and remembers where lesions have already been annotated. (B)~Because uncertainty on this mixed cohort spans many bins, $\lambda^{(t)}$ stays near $\lambda_{\max}$ and removing the gate changes all three metrics by less than one seed standard deviation, as intended (Sec.~\ref{sec:gate}); the gate is a safeguard rather than a source of gain here. Halving or doubling $\lambda_{\max}$ worsens Budget@95\% to 25.2\% and 25.5\%. (C)~Without the floor, $r{=}1$ reduces to uncertainty sampling (29.3\%, SI~0.750) and ``full'' is worse still (30.2\%, SI~0.771), since coverage then admits images whose uncertainty in empty bins is spurious.

\input{tables/table_ablation_components}

\begin{figure}[t]
    \centering
    \includegraphics[width=\columnwidth]{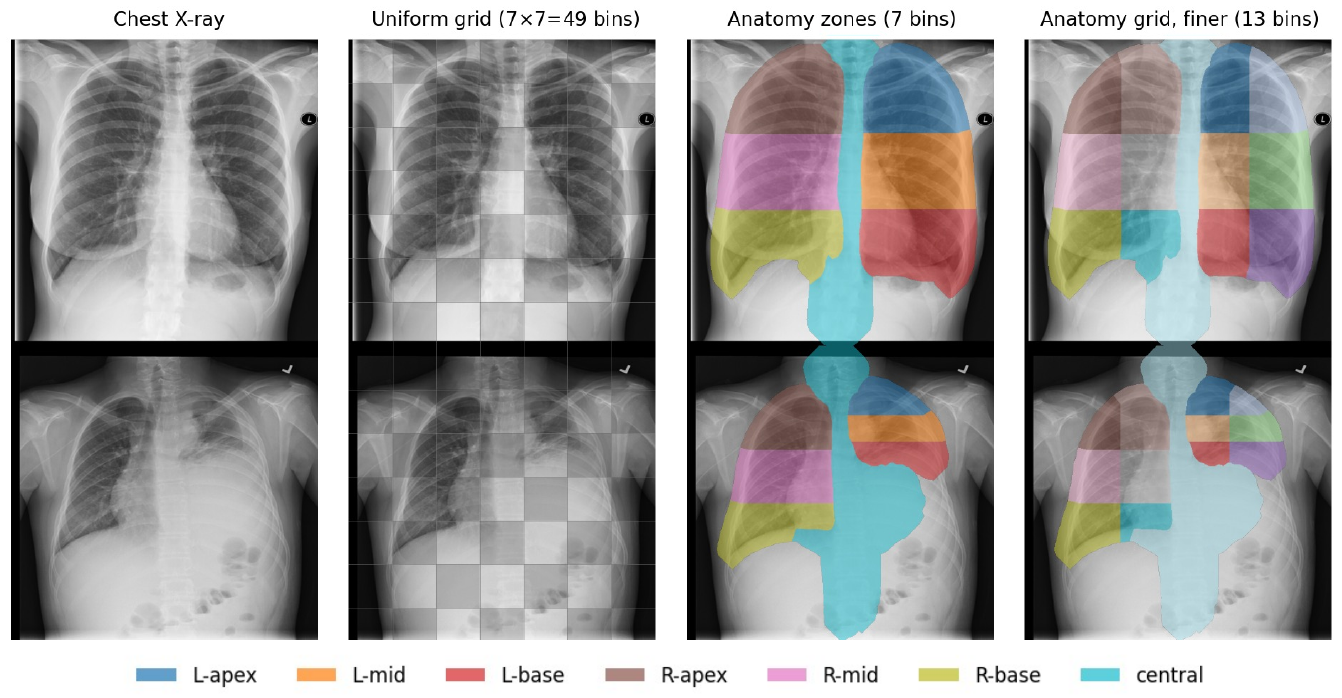}
    \caption{\textbf{Uniform grid vs.\ registration-free anatomy bins.} Rows: two patients; columns: chest X-ray, fixed $7{\times}7$ grid, seven lung-anatomy zones (apex/mid/base per lung + central), and a finer 13-bin within-lung subdivision. Anatomy bins come \emph{per image} from lung segmentation, so a given bin denotes the same anatomical region across patients despite body habitus, rotation, or view.}
    \label{fig:anatomy_bins}
\end{figure}

\input{tables/table_grid_ablation}

\subsubsection{Robustness to Grid Resolution and Evaluation Protocols}
% Table~\ref{tab:integrated_ablation} evaluates \emph{SPARCL}'s sensitivity to grid size ($M$) and evaluation split protocols. While finer bins (e.g., $8 \times 8$) degrade performance via over-fragmentation, the $7 \times 7$ grid achieves the optimal balance, maximizing overall Dice (0.626) and Mean Bin Dice (0.355) while minimizing the required budget (22\%). Crucially, under the challenging \textit{Spatial-Equal} and \textit{Spatial-Dynamic} splits, \emph{SPARCL} consistently maintains the lowest $SI$ and highest $PC$. Notably, in the \textit{Spatial-Equal} protocol, \emph{SPARCL} reaches 95\% performance at a mere 20\% budget---drastically outperforming USIMC (29\%) and Uncertainty (31\%)---proving its robust spatial consistency across diverse evaluation scenarios.
Table~\ref{tab:integrated_ablation} evaluates sensitivity to grid size ($M$) and evaluation split. Finer bins (e.g., $8 \times 8$) degrade performance through over-fragmentation, whereas the $7 \times 7$ grid gives the best overall Dice (0.463) and Mean Bin Dice (0.465) at the lowest budget (24\%). Under the \textit{Spatial-Equal} and \textit{Spatial-Dynamic} splits \emph{SPARCL} retains the lowest $SI$ and highest $PC$, and under \textit{Spatial-Equal} it reaches 95\% performance at a 20\% budget versus 29\% for USIMC and 31\% for Uncertainty.

\subsubsection{Robustness to the Bin Definition: Grid vs.\ Anatomy.}
% Our $7{\times}7$ grid is a deliberately simple, \emph{registration-free} partition: it is spatial rather than anatomical, since a fixed cell need not correspond to the same anatomy across patients. To confirm that \emph{SPARCL} does not hinge on this arbitrary choice, we replace the grid with seven \emph{per-image} lung-anatomy zones (apex/mid/base per lung plus a central region) obtained from an off-the-shelf lung segmentation, ensuring bin $m$ tracks the same anatomical region regardless of body habitus or view (Fig.~\ref{fig:anatomy_bins}). Table~\ref{tab:bin_definition} compares these partitions under an identical protocol. The results are statistically indistinguishable in both label efficiency and spatial reliability. The simple grid matches the anatomy-aware zones \emph{without} introducing an additional segmentation dependency. Therefore, \emph{SPARCL} is agnostic to the bin definition, proving that the reported spatial-reliability gains are not grid artifacts and ensuring applicability to modalities where reliable anatomical parcellation is unavailable.
A fixed $7{\times}7$ cell need not correspond to the same anatomy across patients. To confirm that \emph{SPARCL} does not hinge on this registration-free choice, we replace the grid with seven \emph{per-image} lung-anatomy zones obtained from an off-the-shelf lung segmentation (Fig.~\ref{fig:anatomy_bins}) and compare both partitions under an identical protocol (Table~\ref{tab:bin_definition}). The three are statistically indistinguishable in label efficiency and spatial reliability: the grid matches the anatomy-aware zones without an additional segmentation dependency, so the reported spatial-reliability gains are not grid artifacts and the method remains applicable where anatomical parcellation is unavailable.

\input{tables/table_bin_definition}
\input{tables/table_concave_gain}

% \subsubsection{Robustness to the Concave Gain $\psi$.}
% The coverage gain rewards under-represented bins under \emph{any} concave, non-decreasing $\psi$ with $\psi(0)=0$; Table~\ref{tab:concave_gain} compares three members of this family under an identical protocol. Replacing $\psi(z)=\sqrt{z}$ with $\log(1+z)$ or $z^{1/3}$ moves Budget@95\% from $22\pm2\%$ to $23\pm2\%$ and $24\pm1\%$, Dice@20\% from $43.4\pm2.5$ to $41.8\pm1.9$ and $42.3\pm1.8$, $SI$ from $0.602\pm0.108$ to $0.673\pm0.080$ and $0.608\pm0.051$, and $PC$ from $0.140\pm0.037$ to $0.131\pm0.043$ and $0.140\pm0.015$; every difference lies within one seed standard deviation. \emph{SPARCL} is therefore not sensitive to the specific concave form. We adopt $\sqrt{z}$ because it gives the best budget and Dice among the three and because its steeper slope near the origin most strongly rewards covering empty bins in early rounds.

\subsubsection{Robustness to the Concave Gain $\psi$.}
The coverage gain rewards under-represented bins under \emph{any} concave, non-decreasing $\psi$ with $\psi(0)=0$; Table~\ref{tab:concave_gain} compares three members of this family under an identical protocol. Replacing $\sqrt{z}$ with $\log(1+z)$ or $z^{1/3}$ shifts Budget@95\% by less than two points and leaves Dice@20\%, $SI$, and $PC$ within one seed standard deviation. \emph{SPARCL} is therefore not sensitive to the specific concave form; we adopt $\sqrt{z}$ because it gives the best budget and Dice among the three and its steeper slope near the origin most strongly rewards covering empty bins in early rounds.

\subsubsection{Hyperparameter Sensitivity.}
\emph{SPARCL} exposes two hyperparameters beyond the grid, the gate ceiling $\lambda_{\max}$ and the candidate ratio $r$; both are swept in Table~\ref{tab:ablation_components}~(B)/(C). $r{=}1$ and ``full'' bracket the design---uncertainty sampling on one side, unrestricted cumulative coverage on the other. We fix $\lambda_{\max}=1.0$ and $r=5$ throughout.

%% file: tables/table_fundus.tex
% MODALITY GENERALIZATION — DDR fundus (microaneurysms, diffuse regime)

\begin{table}[t]
\centering
\caption{
\textbf{Modality generalization to color fundus photography} (DDR, microaneurysm
segmentation, three shared seeds), showing the CXR pipeline transfers to a different imaging modality. Dice@20\% and AULC use the annotation-budget fraction; Final Dice is at the last
AL round.}
\label{tab:fundus_generalization}
\scriptsize
\setlength{\tabcolsep}{10pt}
\renewcommand{\arraystretch}{1.15}
\begin{tabular}{lccc}
\toprule
\textbf{Method} & Dice@20\%\,$\uparrow$ & AULC\,$\uparrow$ & Final Dice\,$\uparrow$ \\
\midrule
Random            & 21.0 $\pm$ 0.29 & 25.9 $\pm$ 0.13 & 31.3 $\pm$ 0.21 \\
Area-Random       & 23.4 $\pm$ 0.13 & 26.3 $\pm$ 0.07 & 31.5 $\pm$ 0.14 \\
Uncertainty       & 21.4 $\pm$ 0.12 & 26.4 $\pm$ 0.05 & 31.6 $\pm$ 0.10 \\
Coreset           & \underline{22.9 $\pm$ 0.21} & \underline{26.9 $\pm$ 0.13} & \underline{32.2 $\pm$ 0.11} \\
Diversity         & 15.4 $\pm$ 0.26 & 24.2 $\pm$ 0.21 & 29.4 $\pm$ 0.12 \\
TAUDIS            & 16.5 $\pm$ 0.17 & 22.0 $\pm$ 0.15 & 30.5 $\pm$ 0.13 \\
Loss-Prediction   & 14.9 $\pm$ 0.17 & 22.9 $\pm$ 0.09 & 29.1 $\pm$ 0.10 \\
\rowcolor{Gray} \textbf{SPARCL (Ours)}    & \textbf{26.2 $\pm$ 0.23} & \textbf{29.5 $\pm$ 0.11} & \textbf{32.6 $\pm$ 0.15} \\
\bottomrule
\end{tabular}
\end{table}

%% file: tables/table_ablation_components.tex
\begin{table}[t]
\centering
\caption{\textbf{Ablation of \emph{SPARCL} components and hyperparameters} (in-house calcification, $7\times7$ grid, mean $\pm$ std over five shared seeds). \textbf{(A)} Coverage state: cumulative $H_0$ (Eq.~\eqref{eq:h0}) vs.\ resetting the state every round, and coverage without the uncertainty term. \textbf{(B)} Gate: removing it ($\lambda\equiv\lambda_{\max}$) and varying the ceiling $\lambda_{\max}$. \textbf{(C)} Uncertainty floor: $r{=}1$ is uncertainty sampling (the batch is the top-$B$ by $U(x)$); ``full'' runs the greedy step over the whole pool.}
\label{tab:ablation_components}
\small
\setlength{\tabcolsep}{4pt}
\renewcommand{\arraystretch}{1.1}
\resizebox{\columnwidth}{!}{%
\begin{tabular}{l|c|c|c}
\toprule
\textbf{Configuration} & \textbf{Dice@20\% $\uparrow$} & \textbf{Budget@95\% $\downarrow$} & \textbf{SI $\downarrow$}  \\
\midrule
\multicolumn{4}{l}{\textbf{(A) Coverage state}} \\
\midrule
\textbf{\emph{SPARCL}} (cumulative $H_0$, gate, $r{=}5$) & 43.4 $\pm$ 1.3 & 24.4 $\pm$ 1.4\% & 0.571 $\pm$ 0.038\\
reset ($H_0\equiv0$; within-batch coverage only)         & 41.2 $\pm$ 1.4 & 26.5 $\pm$ 1.6\% & 0.592 $\pm$ 0.012\\
coverage only ($U(x)$ removed)                           & 37.1 $\pm$ 1.6 & 28.7 $\pm$ 1.5\% & 0.624 $\pm$ 0.026\\
\midrule
\multicolumn{4}{l}{\textbf{(B) Positional-diversity gate}} \\
\midrule
w/o gate ($\lambda\equiv\lambda_{\max}$)                 & 43.1 $\pm$ 1.5 & 25.1 $\pm$ 1.9\% & 0.599 $\pm$ 0.034\\
$\lambda_{\max}=0.5$                                     & 43.1 $\pm$ 0.9 & 25.2 $\pm$ 1.2\% & 0.583 $\pm$ 0.018\\
$\lambda_{\max}=2.0$                                     & 42.9 $\pm$ 1.1 & 25.5 $\pm$ 1.1\% & 0.592 $\pm$ 0.021\\
\midrule
\multicolumn{4}{l}{\textbf{(C) Uncertainty floor}} \\
\midrule
$r=1$ (= uncertainty sampling)                           & 41.2 $\pm$ 1.1 & 29.3 $\pm$ 0.9\% & 0.750 $\pm$ 0.045 \\
full (no floor)                                          & 39.1 $\pm$ 0.8 & 30.2 $\pm$ 1.3\% & 0.771 $\pm$ 0.021  \\
\bottomrule
\end{tabular}%
}
\end{table}

%% file: tables/table_grid_ablation.tex
\begin{table}[t]
\centering
\caption{Ablation study on evaluation split strategies and grid size $M$ using the \textit{In-house Calcification} dataset. Dice, $SI$, $PC$, and Mean Bin Dice are measured at the final AL round and averaged over 5 seeds. \emph{SPARCL} consistently maintains superior spatial consistency and label efficiency across all settings.}
\label{tab:integrated_ablation}
\scriptsize
\resizebox{1.0\columnwidth}{!}{%
\begin{tabular}{l|c|c|c|c|c}
\toprule
\multicolumn{2}{c|}{\textbf{Ablation Category}} & \multirow{2}{*}{\textbf{Method}} & \textbf{Budget} & \multirow{2}{*}{\textbf{Dice($\uparrow$)}} & \multirow{2}{*}{\textbf{SI($\downarrow$)}} \\ \cline{1-2}
 {\textbf{Split Strategy}}& \textbf{Grid Size($K \times L$)} & & \textbf{@95\%($\downarrow$)} & &  \\
\hline

\multirow{4}{*}{Random}
& $M=6$~~($3 \times 2$)  & \emph{SPARCL} & 37\% & 0.433 & 0.421  \\
& $M=25$~($5 \times 5$) & \emph{SPARCL} & 28\% & 0.447 & 0.493  \\
& $M=49$~($7 \times 7$) & \emph{\textbf{SPARCL}} & \textbf{24\%} & \textbf{0.463} & \textbf{0.414}  \\
& $M=64$~($8 \times 8$) & \emph{SPARCL} & 24\% & 0.443 & 0.455  \\
\midrule

\multirow{4}{*}{Spatial Equal}
& \multirow{4}{*}{$M=49$~($7 \times 7$)}
& Random & 34\% & 0.445 & 0.511  \\
& & Uncertainty & 31\% & 0.455 & 0.486  \\
& & USIMC~\cite{follmer2024active} & 29\% & 0.444 & 0.520  \\
& & \textbf{\emph{SPARCL}} & \textbf{20\%} & \textbf{0.466} & \textbf{0.409}  \\
\midrule

\multirow{4}{*}{Spatial Dynamic}
& \multirow{4}{*}{$M=49$~($7 \times 7$)}
& Random & 39\% & 0.429 & 0.516  \\
& & Uncertainty & 34\% & 0.432 & 0.505  \\
& & USIMC~\cite{follmer2024active} & 28\% & 0.431 & 0.514  \\
& & \textbf{\emph{SPARCL}} & \textbf{24\%} & \textbf{0.464} & \textbf{0.481}  \\
\bottomrule
\end{tabular}%
}
\end{table}

%% file: tables/table_bin_definition.tex
\begin{table}[t]
\centering
\caption{\textbf{Bin definition: uniform grid vs.\ registration-free anatomy zones} (in-house, $20\%$ budget). Compares the standard $7{\times}7$ grid against anatomical regions (Fig.~\ref{fig:anatomy_bins}). The simple grid maintains competitive spatial reliability without requiring an external lung segmenter. Acquisition uses the partition in each row; SI and PC
  are evaluated on the fixed 7×7 grid for all rows.}
\label{tab:bin_definition}
\small
\setlength{\tabcolsep}{8pt}
\renewcommand{\arraystretch}{1.15}
\resizebox{\columnwidth}{!}{%
\begin{tabular}{lccc}
\toprule
\textbf{Bin scheme} & Dice@20\%\,$\uparrow$ & SI\,$\downarrow$ & PC\,$\uparrow$ \\
\midrule
Uniform grid ($7{\times}7$)      & 43.4 $\pm$ 1.3 & 0.571 $\pm$ 0.038 & 0.148 $\pm$ 0.009 \\
Anatomy zones (7 lung zones)    & 41.8 $\pm$ 3.1 & 0.630 $\pm$ 0.041 & 0.147 $\pm$ 0.008 \\
Anatomy zones (13 within-lung bins)    & 42.4 $\pm$ 2.2 & 0.597 $\pm$ 0.040 & 0.139 $\pm$ 0.013 \\
\bottomrule
\end{tabular}%
}
\end{table}

%% file: tables/table_concave_gain.tex
\begin{table}[t]
\centering
\caption{\textbf{Concave-gain ($\psi$) ablation} (In-house Calcification, Random held-out split, $7\times7$ grid, mean $\pm$ std over 5 seeds). Results are stable across concave functions, confirming the robustness of the diminishing-returns coverage reward.}
\label{tab:concave_gain}
\resizebox{\columnwidth}{!}{%
\begin{tabular}{l|c|c|c|c}
\toprule
\textbf{Gain $\psi(z)$} & \textbf{Budget@95\% $\downarrow$} & \textbf{Dice@20\% $\uparrow$} & \textbf{SI $\downarrow$} & \textbf{PC $\uparrow$} \\
\midrule
$\log(1+z)$          & 26.1 $\pm$ 2.1\% & 42.8 $\pm$ 1.9 & 0.613 $\pm$ 0.080 & 0.131 $\pm$ 0.043 \\
$\sqrt{z}$ (ours)    & 24.4 $\pm$ 1.4\% & 43.4 $\pm$ 1.3 & 0.571 $\pm$ 0.038 & 0.148 $\pm$ 0.009  \\
$z^{1/3}$            & 25.8 $\pm$ 1.5\% & 42.3 $\pm$ 1.8 & 0.608 $\pm$ 0.051 & 0.140 $\pm$ 0.015 \\
\bottomrule
\end{tabular}%
}
\end{table}

%% file: contents/5_conclusions.tex
\section{Conclusion}
In this paper, we proposed \emph{SPARCL}, an active learning framework that combines predictive uncertainty with spatial coverage accumulated over the lesions already labeled, gated off only when the pool offers no positional diversity to exploit. \emph{SPARCL} reaches 95\% of fully supervised performance using only a 24\% annotation budget on our in-house calcification cohort and 21\% on the SIIM-ACR benchmark. Evaluated through our proposed \textit{Spatial Inconsistency} and \textit{Performance Coverage} metrics, \emph{SPARCL} effectively reduces anatomical blind spots that aggregate metrics obscure. Future work will extend this spatially adaptive mechanism to 3D volumetric modalities (e.g., 3D CT) to resolve spatial information loss.

%% file: contents/6_data.tex
\section{Data Availability}
% The in-house calcification dataset (IRB No. 2017-1067, IRB No. B-1704/390-109) is private and cannot be released publicly; it may be made available from the authors upon reasonable request. The public datasets are openly available: VinDr-CXR~\cite{nguyen2022vindr}, SIIM-ACR~\cite{filice2020crowdsourcing}, ChestX-Det10~\cite{liu2020chestx}, and DDR~\cite{li2019diagnostic}. 

The in-house calcification dataset is private and cannot be released publicly; it may be made available from the authors upon reasonable request. The public datasets are openly available: VinDr-CXR~\cite{nguyen2022vindr}, SIIM-ACR~\cite{filice2020crowdsourcing}, ChestX-Det10~\cite{liu2020chestx}, and DDR~\cite{li2019diagnostic}.

%% file: contents/references.tex
% =====================================================================
% IEEE TMI (IEEEtran) 형식으로 변환한 references
% - \cite{key} 는 기존 그대로 사용 가능 (bibitem 키 유지)
% - IEEE 규칙: 저자 6명 이하 -> 전원 표기 / 7명 이상 -> 첫 저자 + et al.
% - 제목은 sentence case + 고유명사/약어 보존 (U-Net, CheXpert 등)
% - arXiv: "YYYY, arXiv:XXXX.XXXXX." 형식
% - 저널 논문번호형: "Art. no. XXX"
% - DOI는 알고 있으면 각 항목 끝에 ", doi: 10.xxxx/xxxx." 로 추가 권장
% =====================================================================

%% file: main.bbl
\section*{REFERENCES}
\begin{thebibliography}{99}

% ---------- Introduction (1st paragraph) ----------
\bibitem{long2015fully} J. Long, E. Shelhamer, and T. Darrell, ``Fully convolutional networks for semantic segmentation,'' in {\it Proc. IEEE Conf. Comput. Vis. Pattern Recognit. (CVPR)}, 2015, pp. 3431--3440.


\bibitem{ronneberger2015u} O. Ronneberger, P. Fischer, and T. Brox, ``U-Net: Convolutional networks for biomedical image segmentation,'' in {\it Proc. Int. Conf. Med. Image Comput. Comput.-Assist. Intervent. (MICCAI)}, 2015, pp. 234--241.

\bibitem{wang2017chestx} X. Wang, Y. Peng, L. Lu, Z. Lu, M. Bagheri, and R. M. Summers, ``ChestX-ray8: Hospital-scale chest X-ray database and benchmarks on weakly-supervised classification and localization of common thorax diseases,'' in {\it Proc. IEEE Conf. Comput. Vis. Pattern Recognit. (CVPR)}, 2017, pp. 2097--2106.

\bibitem{irvin2019chexpert} J. Irvin {\it et al.}, ``CheXpert: A large chest radiograph dataset with uncertainty labels and expert comparison,'' in {\it Proc. AAAI Conf. Artif. Intell.}, vol. 33, no. 1, 2019, pp. 590--597.

\bibitem{budd2021survey} S. Budd, E. C. Robinson, and B. Kainz, ``A survey on active learning and human-in-the-loop deep learning for medical image analysis,'' {\it Med. Image Anal.}, vol. 71, 2021, Art. no. 102062.

\bibitem{nath2020diminishing} V. Nath, D. Yang, B. A. Landman, D. Xu, and H. R. Roth, ``Diminishing uncertainty within the training pool: Active learning for medical image segmentation,'' {\it IEEE Trans. Med. Imag.}, vol. 40, no. 10, pp. 2534--2547, 2020.

\bibitem{smailagic2018medal} A. Smailagic {\it et al.}, ``MedAL: Accurate and robust deep active learning for medical image analysis,'' in {\it Proc. IEEE Int. Conf. Mach. Learn. Appl. (ICMLA)}, 2018, pp. 481--488.

\bibitem{cai2021exploring} L. Cai, X. Xu, L. Zhang, and C.-S. Foo, ``Exploring spatial diversity for region-based active learning,'' {\it IEEE Trans. Image Process.}, vol. 30, pp. 8702--8712, 2021.

\bibitem{nguyen2021goal} C. Nguyen {\it et al.}, ``GOAL: Gist-set online active learning for efficient chest X-ray image annotation,'' in {\it Proc. Med. Imag. Deep Learn. (MIDL)}, 2021, pp. 545--553.

% ---------- Introduction (2nd paragraph) ----------
\bibitem{settles2009active} B. Settles, ``Active learning literature survey,'' Dept. Comput. Sci., Univ. Wisconsin--Madison, Madison, WI, USA, Comput. Sci. Tech. Rep. 1648, 2009.

\bibitem{ren2021survey} P. Ren {\it et al.}, ``A survey of deep active learning,'' {\it ACM Comput. Surv.}, vol. 54, no. 9, pp. 1--40, 2021.

\bibitem{shi2024predictive} J. Shi {\it et al.}, ``Predictive accuracy-based active learning for medical image segmentation,'' 2024, arXiv:2405.00452.

\bibitem{gao2025deep} C. Gao {\it et al.}, ``Deep learning in pulmonary nodule detection and segmentation: A systematic review,'' {\it Eur. Radiol.}, vol. 35, no. 1, pp. 255--266, 2025.

\bibitem{marinakis2024pulmonary} I. Marinakis, K. Karampidis, and G. Papadourakis, ``Pulmonary nodule detection, segmentation and classification using deep learning: A comprehensive literature review,'' {\it BioMedInformatics}, vol. 4, no. 3, pp. 2043--2106, 2024.

\bibitem{gal2016dropout} Y. Gal and Z. Ghahramani, ``Dropout as a Bayesian approximation: Representing model uncertainty in deep learning,'' in {\it Proc. Int. Conf. Mach. Learn. (ICML)}, 2016, pp. 1050--1059.

\bibitem{sener2017active} O. Sener and S. Savarese, ``Active learning for convolutional neural networks: A core-set approach,'' 2017, arXiv:1708.00489.

\bibitem{yang2017suggestive} L. Yang, Y. Zhang, J. Chen, S. Zhang, and D. Z. Chen, ``Suggestive annotation: A deep active learning framework for biomedical image segmentation,'' in {\it Proc. Int. Conf. Med. Image Comput. Comput.-Assist. Intervent. (MICCAI)}, 2017, pp. 399--407.

\bibitem{gorriz2017cost} M. Gorriz, A. Carlier, E. Faure, and X. Giro-i-Nieto, ``Cost-effective active learning for melanoma segmentation,'' 2017, arXiv:1711.09168.

\bibitem{kirsch2019batchbald} A. Kirsch, J. Van Amersfoort, and Y. Gal, ``BatchBALD: Efficient and diverse batch acquisition for deep Bayesian active learning,'' in {\it Adv. Neural Inf. Process. Syst. (NeurIPS)}, vol. 32, 2019.

\bibitem{ash2019deep} J. T. Ash, C. Zhang, A. Krishnamurthy, J. Langford, and A. Agarwal, ``Deep batch active learning by diverse, uncertain gradient lower bounds,'' 2019, arXiv:1906.03671.

\bibitem{siddiqui2020viewal} Y. Siddiqui, J. Valentin, and M. Nie{\ss}ner, ``ViewAL: Active learning with viewpoint entropy for semantic segmentation,'' in {\it Proc. IEEE/CVF Conf. Comput. Vis. Pattern Recognit. (CVPR)}, 2020, pp. 9430--9440

\bibitem{mackowiak2018cereals} R. Mackowiak, P. Lenz, O. Ghori, F. Diego, O. Lange, and C. Rother, ``CEREALS -- Cost-effective region-based active learning for semantic segmentation,'' 2018, arXiv:1810.09726.

\bibitem{wang2016cost} K. Wang, D. Zhang, Y. Li, R. Zhang, and L. Lin, ``Cost-effective active learning for deep image classification,'' {\it IEEE Trans. Circuits Syst. Video Technol.}, vol. 27, no. 12, pp. 2591--2600, 2016.

\bibitem{gal2017deep} Y. Gal, R. Islam, and Z. Ghahramani, ``Deep Bayesian active learning with image data,'' in {\it Proc. Int. Conf. Mach. Learn. (ICML)}, 2017, pp. 1183--1192.

\bibitem{kadir2023edgeal} M. A. Kadir, H. M. T. Alam, and D. Sonntag, ``EdgeAL: An edge estimation based active learning approach for OCT segmentation,'' in {\it Proc. Int. Conf. Med. Image Comput. Comput.-Assist. Intervent. (MICCAI)}, 2023, pp. 79--89.

\bibitem{tian2023fairseg} Y. Tian, M. Shi, Y. Luo, A. Kouhana, T. Elze, and M. Wang, ``FairSeg: A large-scale medical image segmentation dataset for fairness learning using segment anything model with fair error-bound scaling,'' in {\it Proc. Int. Conf. Learn. Represent. (ICLR)}, 2024, pp. 44056--44070.

\bibitem{taha2015metrics} A. A. Taha and A. Hanbury, ``Metrics for evaluating 3D medical image segmentation: Analysis, selection, and tool,'' {\it BMC Med. Imag.}, vol. 15, no. 1, 2015, Art. no. 29.

\bibitem{sudre2017generalised} C. H. Sudre, W. Li, T. Vercauteren, S. Ourselin, and M. Jorge Cardoso, ``Generalised Dice overlap as a deep learning loss function for highly unbalanced segmentations,'' in {\it Proc. Int. Workshop Deep Learn. Med. Image Anal. (DLMIA)}, 2017, pp. 240--248.

\bibitem{reinke2021common} A. Reinke {\it et al.}, ``Common limitations of image processing metrics: A picture story,'' 2021, arXiv:2104.05642.

% ---------- related work
\bibitem{gaillochet2023active} M. Gaillochet, C. Desrosiers, and H. Lombaert, ``Active learning for medical image segmentation with stochastic batches,'' {\it Med. Image Anal.}, vol. 90, 2023, Art. no. 102958.
\bibitem{hacohen2022active} G. Hacohen, A. Dekel, and D. Weinshall, ``Active learning on a budget: Opposite strategies suit high and low budgets,'' in {\it Proc. Int. Conf. Mach. Learn. (ICML)}, 2022, pp. 8175--8195.
\bibitem{yehuda2022active} O. Yehuda, A. Dekel, G. Hacohen, and D. Weinshall, ``Active learning through a covering lens,'' in {\it Adv. Neural Inf. Process. Syst. (NeurIPS)}, vol. 35, 2022, pp. 22354--22367.
\bibitem{xie2022towards} B. Xie, L. Yuan, S. Li, C. H. Liu, and X. Cheng, ``Towards fewer annotations: Active learning via region impurity and prediction uncertainty for domain adaptive semantic segmentation,'' in {\it Proc. IEEE/CVF
Conf. Comput. Vis. Pattern Recognit. (CVPR)}, 2022, pp. 8058--8068.
\bibitem{shin2021all} G. Shin, W. Xie, and S. Albanie, ``All you need are a few pixels: Semantic segmentation with PixelPick,'' in {\it Proc. IEEE/CVF Int. Conf. Comput. Vis. Workshops (ICCVW)}, 2021, pp. 1687--1697.
\bibitem{lin2011class} H. Lin and J. Bilmes, ``A class of submodular functions for document summarization,'' in {\it Proc. 49th Annu. Meeting Assoc. Comput. Linguistics: Human Lang. Technol. (ACL-HLT)}, 2011, pp. 510--520.
\bibitem{maier2024metrics} L. Maier-Hein {\it et al.}, ``Metrics reloaded: Recommendations for image analysis validation,'' {\it Nat. Methods}, vol. 21, no. 2, pp. 195--212, 2024.


% ---------- Experiments: baselines ----------
\bibitem{follmer2024active} B. F{\"o}llmer, K. Schulze, C. Wald, S. Stober, W. Samek, and M. Dewey, ``Active learning with the nnU-Net and sample selection with uncertainty-aware submodular mutual information measure,'' in {\it Proc. Med. Imag. Deep Learn. (MIDL)}, 2024.

\bibitem{yoo2019learning} D. Yoo and I. S. Kweon, ``Learning loss for active learning,'' in {\it Proc. IEEE/CVF Conf. Comput. Vis. Pattern Recognit. (CVPR)}, 2019, pp. 93--102.

\bibitem{yu2023two} K. Yu {\it et al.}, ``Two-step active learning for instance segmentation with uncertainty and diversity sampling,'' 2023, arXiv:2309.16139.

% ---------- Experiments: open datasets ----------
\bibitem{nguyen2022vindr} H. Q. Nguyen {\it et al.}, ``VinDr-CXR: An open dataset of chest X-rays with radiologist's annotations,'' {\it Sci. Data}, vol. 9, no. 1, 2022, Art. no. 429.

\bibitem{filice2020crowdsourcing} R. W. Filice {\it et al.}, ``Crowdsourcing pneumothorax annotations using machine learning annotations on the NIH chest X-ray dataset,'' {\it J. Digit. Imag.}, vol. 33, no. 2, pp. 490--496, 2020.

\bibitem{liu2020chestx} J. Liu, J. Lian, and Y. Yu, ``ChestX-Det10: Chest X-ray dataset on detection of thoracic abnormalities,'' 2020, arXiv:2006.10550.

\bibitem{lian2021structure} J. Lian {\it et al.}, ``A structure-aware relation network for thoracic diseases detection and segmentation,'' {\it IEEE Trans. Med. Imag.}, vol. 40, no. 8, pp. 2042--2052, 2021.

\bibitem{nemhauser1978analysis} G. L. Nemhauser, L. A. Wolsey, and M. L. Fisher, ``An analysis of approximations for maximizing submodular set functions---I,'' {\it Math. Program.}, vol. 14, no. 1, pp. 265--294, 1978.

\bibitem{ma2013submodularity} Y. Ma, R. Garnett, and J. Schneider, ``Submodularity in batch active learning and survey problems on Gaussian random fields,'' 2012, arXiv:1209.3694.

\bibitem{li2019diagnostic} T. Li, Y. Gao, K. Wang, S. Guo, H. Liu, and H. Kang, ``Diagnostic assessment of deep learning algorithms for diabetic retinopathy screening,'' {\it Inf. Sci.}, vol. 501, pp. 511--522, 2019.

\end{thebibliography}
